\documentclass[]{ceurart}

\usepackage{mathtools}
\usepackage{amsthm}
\usepackage{bm}
\usepackage{subcaption}
\usepackage{epstopdf}
\usepackage{xfrac}
\usepackage[shortlabels]{enumitem}
\usepackage{textcomp}
\usepackage{nicefrac}
\usepackage{adjustbox}
\usepackage{pifont}   

\usepackage{algorithm}
\usepackage{algorithmic}

\usepackage[capitalize,noabbrev]{cleveref}

\graphicspath{ {figures/} }

\theoremstyle{plain}
\newtheorem{theorem}{Theorem}[section]

\newtheorem{corollary}[theorem]{Corollary}
\theoremstyle{definition}

\theoremstyle{remark}
\newtheorem{remark}[theorem]{Remark}

\newcommand{\marshall}[2][]{}

\newcommand{\mD}{\mathcal{D}}

\newcommand{\fsim}{f_{\mathrm{sim}}}

\begin{document}

\copyrightyear{2026}
\copyrightclause{Copyright for this paper by its authors.
  Use permitted under Creative Commons License Attribution 4.0
  International (CC BY 4.0).}

\conference{GenAIECommerce'26: The Third Workshop on Agentic and Generative AI
  for E-Commerce, co-located with RecSys, September 28, 2026,
  Minneapolis, MN, USA}

\title{Efficient Clustering with Quality Guardrails for LLM-based Recommender Systems at Industry Scale}

\author[1]{Longshaokan Wang}[email=longsha@amazon.com]
\author[1]{Wai Tsang Keung}[email=wkeung@amazon.com]
\author[1]{Punit Ghodasara}[email=ghopunit@amazon.com]
\author[1]{Roman Wang}[email=wangrz@amazon.com]
\author[1]{Ali Dashti}[email=dashtia@amazon.com]
\author[1]{Francesc Moreno-Noguer}[email=cescmore@amazon.es]
\address[1]{Amazon}

\begin{abstract}
LLMs can be prohibitively expensive and slow to run at scale, especially for applications that invoke an LLM per sample over millions of inputs. A natural way to scale such applications is to cluster the inputs, run the LLM only on cluster representatives, and propagate the LLM outputs to other cluster members. However, in this setup, the LLM outputs a cluster member receives are only as good as the member's match to the cluster representative. Off-the-shelf clustering methods optimize an aggregate objective over all samples, targeting only average-case quality without offering per-sample quality guardrails. As a result, cluster members can be assigned to poorly-matched representatives, and the inherited LLM outputs — though appropriate for the representative — may be irrelevant or even unsafe for the member. For example, a parent of a toddler grouped with parents of older children could receive age-inappropriate recommendations. Furthermore, most clustering methods scale poorly to millions of inputs in runtime and memory, limiting their applicability at industry scale. We propose a scalable two-stage clustering algorithm with provable per-sample quality guardrails: every sample is guaranteed to share a user-specified minimal similarity and exact attribute match with its cluster representative. The algorithm first generates initial clusters with Mini-batch K-Means, then greedily selects representatives within each initial cluster to satisfy the guardrails. We provide theoretical guarantees, complexity analysis, and benchmarks against common clustering methods on internal and public datasets. We show that our method not only delivers per-sample guardrails but also runs substantially faster and scales to data sizes where most standard methods become intractable. We demonstrate our method's impact in a real-world deployment clustering 38 million customers. We reduce downstream LLM cost and runtime by $50\times$ while preserving personalization, unblocking the launch of a persona-based recommender system that delivers significant gains in revenue and engagement in an A/B test.
\end{abstract}


\maketitle

\section{Introduction}
\label{sec:intro}
Modern recommender systems increasingly rely on Large Language Models (LLMs) to deliver richer personalization — from generating personalized search queries and product descriptions to judging recommendation relevance and extracting customer or item attributes \citep{Li2024,Luo2025,Yao2022,Herold2024,Fang2024,Peng2024}. These applications typically invoke an LLM \emph{per user or per item}: each call is individually inexpensive, but the aggregate cost and latency over millions of users or items can be prohibitive at industry scale \citep{Bommasani2022, Erdil2025, Samsi2023, pope2022efficientlyscalingtransformerinference}. Throughput limits imposed by major LLM service providers like Bedrock further stretch end-to-end runtime into days or months \citep{li2025throughputoptimalschedulingalgorithmsllm, Kwon2023, zhu2025nanoflowoptimallargelanguage}, forcing difficult tradeoffs across LLM-based applications. A fundamental challenge of deploying LLMs in a recommender system, then, is to scale per-sample inference within time and budget constraints while satisfying quality guardrails that protect the customer experience.

We investigate using clustering to address LLMs' scalability challenge with quality guardrails. The general scaling approach is to first cluster the input to LLMs, then run the downstream LLMs only for the representative samples from the clusters (e.g., cluster centroids). Clustering can drastically reduce the downstream budget and time required, but introduces approximation errors by assuming that the downstream output for any sample would be the same as the output for the corresponding cluster representative. To ensure that the output for a representative sample is relevant to all the other samples from the same cluster, it is crucial to measure and impose guardrails on within-cluster similarity during clustering.

In particular, we address a real-world application of generating personalized recommendations from customer shopping personas with LLMs. At a high level, the recommender system would consume the customer shopping personas derived from interaction histories, generate personalized search queries from the personas with an LLM, retrieve the corresponding products with a Query-to-Product service, and filter the retrieved products with a fine-tuned LLM-as-a-Judge called Marketing Critic (Figure \ref{fig:crm}). Before clustering is introduced, to generate recommendations for 38 million customers required for the initial launch, the query generation step would take \$114,000 and 23 days, and the Marketing Critic step would take \$1,018,500 and 485 days with our allocated computation resources. Our goal is to cluster the customer personas and estimated household attributes before the query generation step to reduce the downstream budget and time required while protecting the quality of the final recommendations. The quality guardrails are particularly important in eCommerce and can be viewed as safety constraints, because irrelevant product recommendations can damage customer trust and in the worst case cause safety hazards (e.g., recommending age-inappropriate toys to young children).

\begin{figure}[t]
  \centering
  \includegraphics[width=1.0\textwidth]{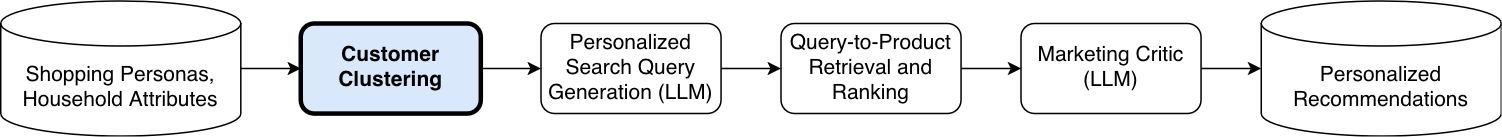}
  \caption{High-level overview of the personalized recommendation pipeline.}
  \label{fig:crm}
\end{figure}

As the recommendations are personalized to the customer personas and attributes like gender and child presence, and we need to generate the recommendations at a large scale with periodic refreshes, we adopt the following requirements when researching the appropriate clustering algorithm to account for both quality guardrails and scalability:
\begin{enumerate}
    \item The semantic similarity between any customer persona in a cluster and the representative persona of the cluster exceeds a user-specified threshold.
    \item All customers in a cluster share exactly the same user-specified attributes.
    \item The number of clusters is meaningfully less than the number of customers (a minimal 10-fold reduction required in our target use case).
    \item The clustering algorithm scales well to large datasets (e.g., 38 million customers) regarding runtime, memory, and cost. The runtime and cost of clustering are significantly less than those of the downstream computation saved due to clustering.
\end{enumerate}
To our knowledge, there is no existing clustering algorithm that satisfies all the above requirements. We propose Scalable Guardrailed Clustering (SGC), a novel two-stage clustering algorithm with greedy cluster representative selection that meets all the requirements: We use Mini-batch K-Means \citep{sculley2010minibatchkmeans} to generate the initial clusters, compute pairwise semantic similarity within each initial cluster, then further partition the initial clusters based on the user-specified semantic similarity threshold and customer attributes as guardrails via an iterative greedy selection strategy to obtain the final clusters.

We review the related work in Section~\ref{sec:related}; describe the proposed clustering algorithm in Section~\ref{sec:method}; benchmark on internal and public datasets in Section~\ref{sec:exp}; show the application on the persona-based personalized recommendation use case in Section~\ref{sec:app}; and summarize and discuss future work in Section~\ref{sec:disc}.

The contributions of this paper are as follows:
\begin{enumerate}
    \item We propose a simple and scalable clustering algorithm that guarantees a user-specified minimal semantic similarity between each sample and its cluster representative and exact matching of user-specified attributes as quality guardrails to protect the customer experience.
    \item We show that the proposed greedy selection step enables effective trimming of tail clusters to further improve data reduction at the tradeoff of removing a small percentage of samples.
    \item We benchmark the proposed method against common clustering methods on both internal and public datasets.
    \item We demonstrate how the proposed clustering algorithm applies to the use case of persona-based personalized recommendations and achieves a 50-fold reduction in downstream computation time and budget to unblock the launch.
\end{enumerate}


\section{Related Work}
\label{sec:related}
Existing clustering methods fail to address all the requirements of Section \ref{sec:intro} simultaneously. Most methods, such as K-Means \citep{macqueen1967multivariate}, BIRCH \citep{zhang1996birch}, and Gaussian Mixture \citep{dempster1977gmm}, optimize an aggregate objective and cannot guarantee a minimal similarity between each sample and its representative or exact attribute matching. For our use case, such mismatches can surface wrong personalized recommendations and hurt the customer experience. Fuzzy hashing methods like MinHash \citep{broder1997minhash} rely on lexical similarity, which poorly captures semantic similarity and cannot handle paraphrases.

The clustering methods most related to ours are \emph{threshold-based}: they impose a user-specified lower bound on within-cluster similarity. Star Clustering \citep{aslam1998static}, SimClus \citep{alhasan2011}, and the Community Detection routine in SentenceTransformers \citep{reimers-2019-sentence-bert} all select representatives on a thresholded similarity graph, but each computes the full $O(n^2)$ pairwise similarity matrix over the entire dataset and thus does not scale to sample size $n \sim 10^7$ (Section \ref{sec:complexity}), our central bottleneck. They also have method-specific drawbacks: Star Clustering's degree-based selection yields substantially more clusters (weaker data reduction) than coverage-based selection, and Community Detection does not strictly guarantee the threshold. SimClus is the closest: our method can be viewed as an extension of the greedy \textsc{Set Cover} selection \citep{johnson1974setcover,chvatal1979setcover} of SimClus by adding an initial round of clustering for scalability, exact attribute matching, and an optional reassignment step to improve average within-cluster similarity. Finally, SemDedup \citep{abbas2023semdedupdataefficientlearningwebscale} shares our cluster-then-compare structure but is a \emph{deduplication} method that discards near-duplicates without retaining a sample-to-representative mapping required for output propagation from each cluster representative to its members. We provide an extended comparison of all four methods and a discussion of complementary LLM-scaling approaches in Appendix \ref{app:related-extended}.

\section{Method}
\label{sec:method}
We formulate the clustering problem, present the main algorithm, formalize the theoretical guarantees, provide the computational complexity analysis, then share the key technical details and design considerations.

\subsection{Problem Formulation}
Let $\mD = \{(P_i, \mathbf{A}_i)\}_{i=1, 2, \ldots, n}$ be the input data to clustering from $n$ customers, where $P_i$ and $\mathbf{A}_i$ are the shopping personas and attributes of customer $i$ respectively. Let $f_\text{emb}$ be an embedding model that converts a textual persona $P_i$ into a numeric vector $\mathbf{E}_i$. Let $f_\text{sim}$ be the Cosine similarity function. The goal of clustering is to learn a mapping from customer data to cluster ID and find a representative sample for each cluster. Equivalently, the goal is to learn a mapping from customer data to a representative customer ID: $\widehat{f}_\text{cluster}: \mD \rightarrow \{1, 2, \ldots, n\}$. The minimal similarity and attribute matching requirements can be defined as: $\widehat{f}_\text{cluster}(P_i, \mathbf{A}_i) = j$ only if $f_\text{sim}\left(f_\text{emb}(P_i), f_\text{emb}(P_j)\right) \geq \alpha$ for a user-specified threshold $\alpha$ and $\mathbf{A}_i = \mathbf{A}_j$. Note that here we formulate the clustering problem for the personalized recommendation use case without loss of generality --- $\mathbf{E}_i$ can be replaced with any feature vector and $f_\text{sim}$ can be replaced with any distance or similarity function.

\subsection{Algorithm}
\label{sec:algo}

\begin{figure}[t!]
  \centering
  \includegraphics[width=1.0\textwidth, trim={0 0cm 0 0cm}]{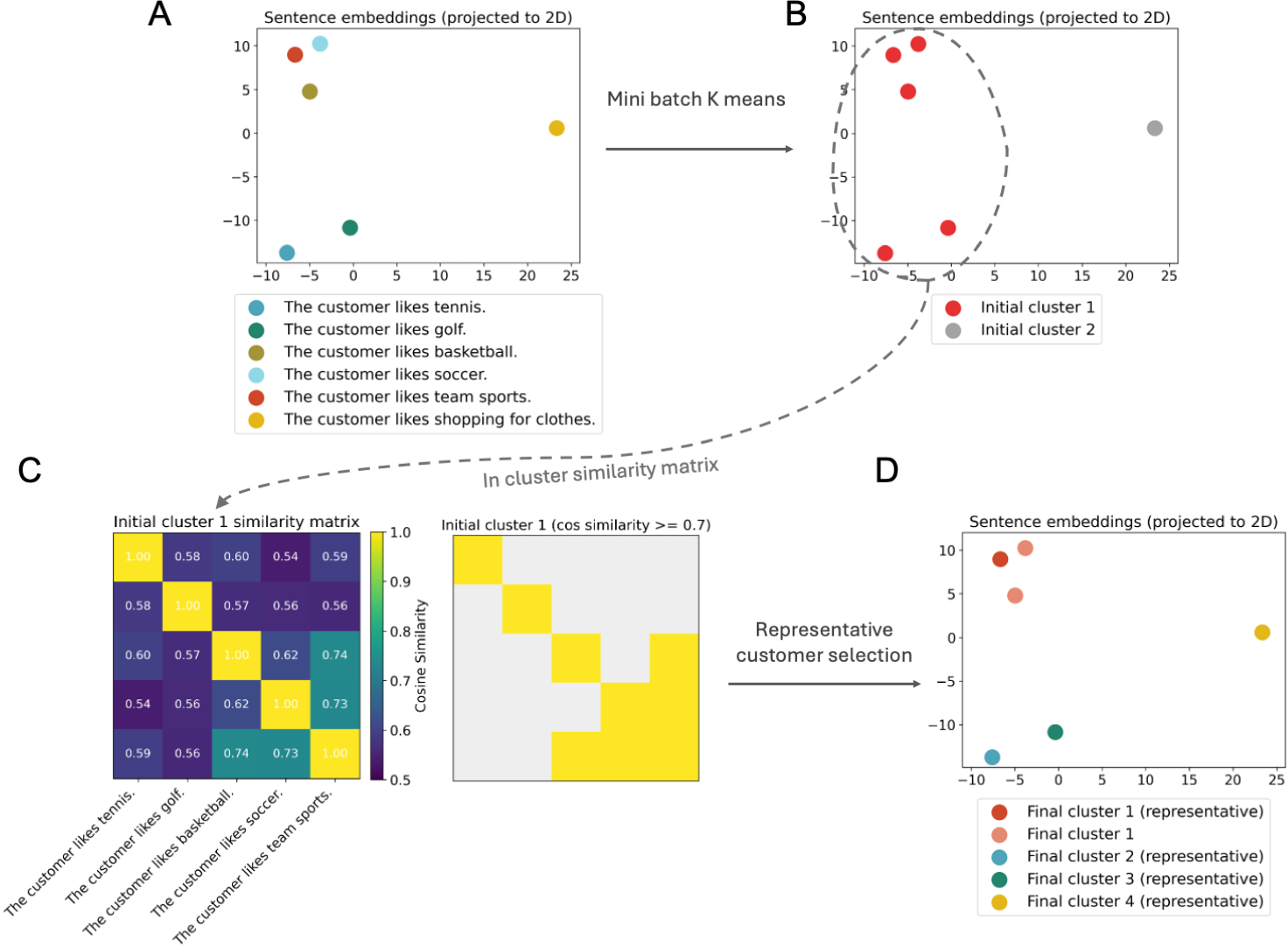}
  \caption{A. Visualization of the embeddings of six sentences projected to a 2D plot. B. The six embeddings are clustered into two initial clusters by Mini-batch K-Means. C. Cosine similarity matrix of initial cluster 1 (left) and after filtering with a similarity threshold of 0.7 (right). ``The customer likes team sports'' has the highest number of matches and is selected as the representative. D. The final clusters.}
  \label{fig:walkthrough}
\end{figure}

Scalable Guardrailed Clustering (SGC) proceeds in two stages. In the first stage, we perform an initial round of clustering on the persona embeddings with Mini-batch K-Means; in the second stage, we compute pairwise similarity and match matrices within each initial cluster and iteratively select the customer with the most matches as a cluster representative. The detailed steps are listed below and illustrated in Figure \ref{fig:walkthrough}, and the pseudocode is provided in Algorithm \ref{algo1}:
\begin{enumerate}
    \item Convert all customer personas to embeddings with a language model (Figure \ref{fig:walkthrough}A).
    \item Generate the \textsl{initial} clusters of the embeddings with Mini-batch K-Means (Figure \ref{fig:walkthrough}B). The subsequent Steps 3-6 are performed within each initial cluster.
    \item Within an initial cluster, compute the pairwise Cosine similarity among all persona embeddings (Figure \ref{fig:walkthrough}C).
    \item We consider 2 customers to be a ``match'' if the Cosine similarity of their persona embeddings exceeds a user-specified threshold (Figure \ref{fig:walkthrough}C) and they share the same user-specified attributes (e.g., household compositions). Find the customer that ``matches'' with the largest number of customers.
    \item Set the customer with the most matches from Step 4 to be the ``representative'' of all its matched customers. The matched customers are considered to be in the same final cluster.
     \item For the remaining unmatched customers, repeat Steps 4-5 until all customers are matched with a representative customer. This effectively further partitions an initial cluster from Step 2 into final clusters defined by the representative customers.
     \item Repeat Steps 3-6 for each initial cluster to generate the final clusters for all customers (Figure \ref{fig:walkthrough}D).
\end{enumerate}

\begin{algorithm*}[t]
\caption{Two-Stage Clustering with Greedy Representative Selection (SGC)}
\label{algo1}
\begin{algorithmic}[1]
\REQUIRE Dataset $\mathcal{D} = \{(P_i, \mathbf{A}_i)\}_{i=1}^{n}$, embedding function $f_\text{emb}$, similarity function $f_\text{sim}$, similarity threshold $\alpha$, number of initial clusters $K$.
\ENSURE Cluster assignment $\widehat{f}_\text{cluster}: \mathcal{D} \rightarrow \{1, 2, \ldots, n\}$; data reduction $|\widehat{f}_\text{cluster}(\mathcal{D})| < |\mathcal{D}|$; minimal similarity $\widehat{f}_\text{cluster}(P_i, \mathbf{A}_i) = j \Rightarrow f_\text{sim}\left(f_\text{emb}(P_i), f_\text{emb}(P_j)\right) \geq \alpha$; attribute matching $\widehat{f}_\text{cluster}(P_i, \mathbf{A}_i) = j \Rightarrow \mathbf{A}_i = \mathbf{A}_j$.

\STATE \textbf{Stage 1: Initial Clustering}
\STATE $\mathbf{E}_i \leftarrow f_\text{emb}(P_i) \; \forall \; i \in \{1, 2, \ldots, n\}$ \COMMENT{Convert personas to embeddings}
\STATE $\mathcal{C}_\text{init} \leftarrow \text{MiniBatchKMeans}(\{\mathbf{E}_i\}_{i=1}^{n}, K)$ \COMMENT{Generate initial clusters}

\STATE \textbf{Stage 2: Representative Customer Selection}
\FOR{each initial cluster $C_k \in \mathcal{C}_\text{init}$}
    \STATE $S_{ij} \leftarrow f_\text{sim}(\mathbf{E}_i, \mathbf{E}_j) \; \forall \; i, j \in C_k$ \COMMENT{Compute pairwise similarities within $C_k$}
    \STATE $M_{ij} \leftarrow \mathbf{I}\{S_{ij} \geq \alpha \; \& \; \mathbf{A}_i = \mathbf{A}_j\} \; \forall \; i, j \in C_k$ \COMMENT{Compute pairwise matches within $C_k$}
    \STATE $\mathcal{U} \leftarrow C_k$ \COMMENT{Set of unmatched customers}
    \STATE $\mathcal{R} \leftarrow \emptyset$ \COMMENT{Set of representative customers}

    \WHILE{$\mathcal{U} \neq \emptyset$}
        \STATE $\mathcal{M}_i \leftarrow \{j \in \mathcal{U} : M_{ij} = 1\} \; \forall \; i \in C_k$ \COMMENT{Find matches from the unmatched customers}
        \STATE $r^* \leftarrow \arg\max_{i \in C_k} |\mathcal{M}_i|$ \COMMENT{Customer with the most matches}
        \STATE $\mathcal{R} \leftarrow \mathcal{R} \cup \{r^*\}$ \COMMENT{Add to representatives}
        \STATE $\widehat{f}_\text{cluster}(P_i, \mathbf{A}_i) \leftarrow r^* \; \forall \; i \in \mathcal{M}_{r^*}$ \COMMENT{Assign matched customers to the representative}
        \STATE $\mathcal{U} \leftarrow \mathcal{U} \setminus \mathcal{M}_{r^*}$ \COMMENT{Remove matched customers}
    \ENDWHILE
\ENDFOR
\RETURN $\widehat{f}_\text{cluster}$
\end{algorithmic}
\end{algorithm*}

Note that the iterative greedy selection of representatives based on the largest match count will always identify the largest final cluster possible for the remaining unmatched customers within an initial cluster. Consequently, the cluster size distribution of the final clusters is heavily skewed with a long tail by design. The skewed cluster size distribution is an intended benefit of the proposed algorithm --- Section \ref{sec:cluster_size} shows that we can remove a large percentage of the tail clusters at the cost of removing a small percentage of customers, further improving the effectiveness of data size reduction; we can also focus on the largest clusters that cover the majority of customers during quality assurance.

One drawback of the greedy representative selection is that a customer can get matched to a representative customer early in the process while having a higher similarity with a later selected representative customer. To further improve the average within-cluster similarity, we propose a variation, SGC with reassignment (SGC-R), applied after Step 6: While keeping the selected representative customers in an initial cluster fixed, we reassign each customer to the representative customer with the highest similarity and matching attributes within that initial cluster (Algorithm~\ref{algo2}). This can be done efficiently with minimal additional latency as we have already computed the similarity and match matrices in Steps 3-4. This variation with reassignment improves within-cluster similarity at the tradeoff of making the cluster size distribution less skewed and thus tail cluster trimming less effective compared to the original algorithm,
as shown in Sections \ref{sec:sim} and \ref{sec:cluster_size}.
Importantly, both SGC and SGC-R satisfy all the desired requirements listed in Section \ref{sec:intro}. More design considerations and technical details are discussed in Section \ref{app:design}.

\begin{algorithm*}[t]
\caption{Reassignment Step (optionally applied within Algorithm \ref{algo1}, Stage 2, for each initial cluster)}
\label{algo2}
\begin{algorithmic}[1]
\REQUIRE Initial cluster $C_k$, representative set $\mathcal{R} \subseteq C_k$, similarity matrix $\{S_{ij}\}_{i,j \in C_k}$, match matrix $\{M_{ij}\}_{i,j \in C_k}$, cluster assignment $\widehat{f}_\text{cluster}$ from Algorithm \ref{algo1}.
\ENSURE Updated cluster assignment $\widehat{f}_\text{cluster}$ with improved average within-cluster similarity; minimal similarity and attribute matching guarantees preserved.

\FOR{each customer $i \in C_k$}
    \STATE $r^* \leftarrow \arg\max_{r \in \mathcal{R} : M_{ir} = 1} S_{ir}$ \COMMENT{Find the best-matching representative}
    \IF{$r^*$ exists}
        \STATE $\widehat{f}_\text{cluster}(P_i, \mathbf{A}_i) \leftarrow r^*$ \COMMENT{Reassign customer to the most similar representative}
    \ENDIF
\ENDFOR
\RETURN $\widehat{f}_\text{cluster}$
\end{algorithmic}
\end{algorithm*}

\subsection{Theoretical Guarantees}
\label{sec:theory}

We formalize the guardrail property and prove that both algorithm
variants satisfy it. Define the \emph{match relation}
$i \sim_\alpha j$ to hold iff $\fsim(\mathbf{E}_i, \mathbf{E}_j) \geq \alpha$
and $\mathbf{A}_i = \mathbf{A}_j$, and let the \emph{$\alpha$-ball} of
$i$ be $\mathcal{B}_\alpha(i) = \{\, j : i \sim_\alpha j \,\}$. The
relation is reflexive (since $\fsim(\mathbf{E}_i, \mathbf{E}_i) = 1$) and
symmetric, but not transitive in general. A cluster assignment
$\widehat{f}_{\text{cluster}}$ with representative set $\mathcal{R}$
\emph{satisfies the guardrail property} if
$\widehat{f}_{\text{cluster}}(P_i, \mathbf{A}_i) = r$ implies
$r \in \mathcal{R}$ and $i \in \mathcal{B}_\alpha(r)$ --- i.e., every
customer is assigned to a representative whose $\alpha$-ball covers
them. This is exactly Requirements~1 and~2 of Section~\ref{sec:intro}.

Within each initial cluster $C_k$, Stage~2 of Algorithm~\ref{algo1} is
the classical Johnson--Chv\'atal greedy heuristic for \textsc{Set Cover}
\citep{johnson1974setcover,chvatal1979setcover} applied to the universe
$C_k$ with candidate sets $\{\mathcal{B}_\alpha(i) \cap C_k\}_{i \in C_k}$:
at each iteration it picks the customer covering the most still-uncovered
elements and removes them.

\begin{theorem}[Guardrail guarantee]
\label{thm:guardrail}
For any input $\mathcal{D}$, any symmetric similarity function
$\fsim$ with $\fsim(x,x) \geq \alpha$, any threshold $\alpha$,
and any number of initial clusters $K \geq 1$:
\begin{enumerate}[label=(\roman*)]
  \item The assignment returned by Algorithm~\ref{algo1} satisfies the
  guardrail property.
  \item Applying Algorithm~\ref{algo2} after Algorithm~\ref{algo1}
  preserves the guardrail property and, for every customer $i$, weakly
  increases the similarity to their assigned representative:
  $\fsim(\mathbf{E}_i, \mathbf{E}_{r^{\mathrm{new}}}) \geq
   \fsim(\mathbf{E}_i, \mathbf{E}_{r^{\mathrm{old}}})$.
\end{enumerate}
\end{theorem}

\begin{proof}[Proof sketch]
\textbf{(i)} Fix any initial cluster $C_k$, and let $\mathcal{U}^{(t)}$
denote the unmatched set at iteration $t$ of Stage~2. By the definition
$M_{ij} = \mathbf{1}\{S_{ij} \geq \alpha \text{ and } \mathbf{A}_i = \mathbf{A}_j\}$
(line~7 of Algorithm~\ref{algo1}),
$\mathcal{M}_i^{(t)} = \mathcal{B}_\alpha(i) \cap \mathcal{U}^{(t)}$.
Every customer $j$ assigned to $r^{*(t)}$ in line~14 lies in
$\mathcal{M}_{r^{*(t)}}^{(t)} \subseteq \mathcal{B}_\alpha(r^{*(t)})$,
so $\fsim(\mathbf{E}_j, \mathbf{E}_{r^{*(t)}}) \geq \alpha$ and
$\mathbf{A}_j = \mathbf{A}_{r^{*(t)}}$. Once assigned, $j$ is removed
from $\mathcal{U}$ (line~15) and never reassigned. Reflexivity ensures
$r^{*(t)} \in \mathcal{M}_{r^{*(t)}}^{(t)}$, so at least one element is
removed per iteration and the loop terminates in at most $|C_k|$
iterations with every $i \in C_k$ assigned to some
$r \in \mathcal{B}_\alpha(i) \cap C_k$. Stage~1 partitions
$\{1,\ldots,n\}$ into $\{C_k\}$ and Stage~2 is applied independently
within each, so the guarantee extends to all of $\mathcal{D}$.

\textbf{(ii)} Let $\mathcal{R}_k$ be the representatives selected by
Algorithm~\ref{algo1} within $C_k$. By (i), the original representative
$r^{\mathrm{old}} \in \mathcal{R}_k$ satisfies
$i \in \mathcal{B}_\alpha(r^{\mathrm{old}})$, equivalently
$r^{\mathrm{old}} \in \mathcal{B}_\alpha(i)$ by symmetry of $\sim_\alpha$.
Hence $\mathcal{R}_k \cap \mathcal{B}_\alpha(i)$ is nonempty and the
$\arg\max$ in line~2 of Algorithm~\ref{algo2} is over a nonempty set
that includes $r^{\mathrm{old}}$. The selected $r^{\mathrm{new}}$ thus
lies in $\mathcal{B}_\alpha(i)$ (guardrail preserved) and satisfies
$\fsim(\mathbf{E}_i, \mathbf{E}_{r^{\mathrm{new}}}) \geq
 \fsim(\mathbf{E}_i, \mathbf{E}_{r^{\mathrm{old}}})$ since
$r^{\mathrm{old}}$ is feasible in the $\arg\max$.
\end{proof}

We emphasize that Theorem~\ref{thm:guardrail} holds for any Stage~1
partition, including $K=1$; the choice of initial clustering affects
only the number of final clusters and runtime, not correctness. The
classical Johnson--Chv\'atal analysis further bounds
$|\mathcal{R}_k| \leq (1 + \ln |C_k|) \cdot \mathrm{OPT}_k$, where
$\mathrm{OPT}_k$ is the minimum cover within $C_k$, providing a formal
ceiling on the data-reduction loss from greedy selection. Full proofs
and additional analysis are provided in Appendix~\ref{app:formal-guarantees}.

\subsection{Computational Complexity}
\label{sec:complexity}

We analyze the time and memory complexity of Algorithms~\ref{algo1}
and~\ref{algo2} as a function of dataset size $n$, embedding dimension
$d$, and number of initial clusters $K$. Let $n_k = |C_k|$ with
$\sum_k n_k = n$.

\paragraph{Stage 1 (Mini-batch K-Means).}
With user-specified batch size $b$ and $T_1$ iterations, Stage~1 runs
in $O(T_1 \, b \, K \, d)$ time and $O((n + K) d)$ memory
\citep{sculley2010minibatchkmeans}. Both are effectively linear in $n$
since $b, T_1$ are constants independent of $n$.

\paragraph{Stage 2 (representative selection within each $C_k$).}
Computing the pairwise similarity and match matrices costs
$O(n_k^2 \, d)$ time and $O(n_k^2)$ memory. The greedy loop maintains
row-sums of $M$ restricted to the unmatched set; since the matched sets
$\mathcal{M}_{r^{*(t)}}^{(t)}$ partition $C_k$, the total update cost
across all iterations is $O(n_k^2)$. Stage~2 thus costs $O(n_k^2 \, d)$
time and $O(n_k^2)$ memory per cluster, and processing one initial
cluster at a time gives peak memory $O(\max_k n_k^2)$.

\paragraph{Stage 2b (reassignment).}
Algorithm~\ref{algo2} reuses the already-materialized $S$ and $M$ and
costs $O(n_k \, |\mathcal{R}_k|) \leq O(n_k^2)$ per cluster, dominated
by Stage~2 and adding no asymptotic overhead.

\paragraph{Aggregate.}
For roughly balanced initial clusters with $n_k \approx n/K$,
$$
T_{\text{total}} = O\!\left(n d + \tfrac{n^2 d}{K}\right),
\quad
M_{\text{total}} = O\!\left(n d + \tfrac{n^2}{K^2}\right).
$$
Thus increasing $K$ reduces Stage~2 time linearly and peak memory
quadratically, at the cost of slightly reduced data-reduction
efficiency (Appendix~\ref{app:robust}). In the regime $K = \Theta(n)$,
the entire algorithm becomes linear in $n$. See Appendix~\ref{app:complexity} for the full derivation.

\paragraph{Comparison with baselines.}
Table~\ref{tab:complexity} contrasts SGC with the
baselines benchmarked in Section~\ref{sec:sim}. The average-metric methods
(K-Means, Gaussian Mixture, etc.) do not admit a within-cluster similarity
constraint, and methods with $O(n^2)$ memory (Agglomerative, Spectral) become
prohibitive at industry scale: e.g., Agglomerative requires ${\sim}90$\,TB at
$n$=5M, far beyond a single instance. The threshold-based methods most related
to ours --- Star Clustering, SimClus, and Community Detection --- all build the full
$O(n^2)$ pairwise similarity matrix over the entire dataset, incurring $O(n^2 d)$
time and $O(n^2)$ memory and hitting the same scalability wall (Section
\ref{app:comp-comp}). Algorithm~\ref{algo1} sidesteps this by restricting the
quadratic computation to within each initial cluster, yielding
$O(\max_k n_k^2) \ll O(n^2)$ memory for any $K \gg 1$ --- the underlying reason
for the up-to-$1000\times$ runtime and memory advantage observed in
Section~\ref{sec:sim}.

\begin{table}[t]
  \caption{Asymptotic complexity comparison. $n$: dataset size, $d$:
embedding dimension, $K$: number of (final) clusters or initial
clusters for SGC, $T$: iterations, $b$: mini-batch
size. The optional reassignment step (Alg.~\ref{algo2}) does not
change the asymptotic complexity. Baselines do not support a
within-cluster similarity guardrail.}
  \label{tab:complexity}
  \centering
  \adjustbox{width=0.7\columnwidth}{
  \begin{tabular}{lll}
    \toprule
    Method & Time & Memory \\
    \midrule
    \textbf{SGC (Alg.~\ref{algo1})} & $O(nd + n^2 d / K)$ & $O(nd + n^2/K^2)$ \\
    \midrule
    Star Clustering & $O(n^2 d)$ & $O(n^2)$ \\
    SimClus & $O(n^2 d)$ & $O(n^2)$ \\
    Community Detection & $O(n^2 d)$ & $O(n^2)$ \\
    \midrule
    K-Means & $O(n K d T)$ & $O(nd)$ \\
    Mini-batch K-Means & $O(b K d T)$ & $O(nd)$ \\
    Agglomerative (Ward) & $O(n^2 d)$ & $O(n^2)$ \\
    BIRCH & $O(nd)$ amortized & $O(nd)$ \\
    Spectral & $O(n^3)$ & $O(n^2)$ \\
    Gaussian Mixture (EM) & $O(n K d^2 T)$ & $O(nd + K d^2)$ \\
    \bottomrule
  \end{tabular}
  }
\end{table}

\subsection{Design Considerations and Technical Details}
\label{app:design}
By design, SGC satisfies all the desired requirements listed in Section \ref{sec:intro}: Steps 4-5 ensure that any customer and their matched representative share a persona similarity above the user-specified threshold and the same user-specified attributes (Section \ref{sec:theory}). The algorithm scales to large datasets (38 million samples in our recommendation use case in Section \ref{sec:app}) because the embedding, Mini-batch K-Means, and greedy selection steps are all cheap, and the expensive pairwise similarity computation is confined to within each initial cluster (Section \ref{sec:complexity}). Section \ref{sec:sim} shows the speed advantage over benchmark methods, and Section \ref{sec:tradeoff} shows that even at a relatively high similarity threshold of 0.75, we still achieve 186-fold data reduction.

We cluster on persona embeddings to capture semantic similarity, a better approach than lexical near-deduplication like MinHash \citep{broder1997minhash}, which cannot handle paraphrases. We select an embedding model from the SentenceTransformer model zoo \citep{reimers-2019-sentence-bert} based on three criteria: the outputs support both Euclidean distance (used in initial clustering) and Cosine similarity (used in representative selection), the model performs well across benchmark datasets, and the latency is reasonable. Based on these criteria and the input token count of our shopping persona data, we select \texttt{all-MiniLM-L6-v2} for the recommendation use case, though any model meeting the criteria can be used.

The initial clustering need not be precise, since the subsequent representative selection will further partition it. We thus choose Mini-batch K-Means, which is significantly faster and more memory-efficient than standard K-Means while achieving similar results \citep{sculley2010minibatchkmeans}. The further partition step also makes SGC robust to its hyperparameters (Appendix~\ref{app:robust}). The number of initial clusters $K$, however, controls a meaningful tradeoff: increasing $K$ reduces Stage~2 time linearly and peak memory quadratically (Section~\ref{sec:complexity}), but also slightly reduces data-reduction efficiency by potentially separating similar customers into different initial clusters. On 100K shopping personas, varying $K$ from 10 to 250 changes the final cluster count by only ${\sim}50\%$ while keeping minimal within-cluster similarity at the user-specified threshold (Table~\ref{tab:robustness}). In practice we recommend the smallest $K$ that latency and memory budgets allow.

Between periodic re-clustering, new customers can be assigned to existing representatives whose $\alpha$-balls cover them, or form new clusters otherwise.

\section{Experiments}
\label{sec:exp}

In Section \ref{sec:sim}, we show empirical evidence that the proposed clustering algorithm can guarantee the user-specified minimal similarity between any sample and the cluster representative while having a significant speed advantage over benchmarks.\footnote{A note on reproducibility: while we cannot release the codebase, the internal persona dataset, or the business metrics from the A/B test, we provide pseudocode to exactly reproduce the proposed method (Algorithms~\ref{algo1} and~\ref{algo2}), the experimental setup and hyperparameters (Appendix~\ref{app:hyp}), benchmarks on public datasets (Section~\ref{sec:sim}), and the exact latency and cost reductions from the deployment (Section~\ref{sec:app}).} In Section \ref{sec:cluster_size}, we show that the proposed algorithm results in more skewed cluster size distribution than benchmarks, enabling effective trimming of tail clusters to further improve data size reduction. In Section \ref{sec:rel}, we validate the intuition that higher within-cluster similarity leads to more relevant product recommendations. In Section \ref{sec:tradeoff}, we investigate the impact of required minimal similarity and household attributes on the number of final clusters to quantify the tradeoff between personalization and data reduction. Additional details of the experimental setup are provided in Appendix~\ref{app:hyp}.

We benchmark on both the internal dataset of shopping personas and the public datasets of AG News \citep{DelCorso2005,Gulli2005} and Cosmopedia \citep{benallal2024cosmopedia}. The full internal dataset contains 38 million customers' shopping personas and estimated household attributes derived from interaction histories. We take a random subset of 100K samples from the internal data to speed up experiments, especially when certain benchmark methods, such as Agglomerative Clustering and Spectral Clustering, cannot scale to large datasets (e.g., Agglomerative Clustering's $O(n^2)$ memory requirement translates to ${\sim}90$\,TB on just 5 million customers; see Section~\ref{sec:complexity}). For public data we use the training split of AG News with 120K samples and Cosmopedia-100K of 100K samples.

\subsection{Within-Cluster Similarities}
\label{sec:sim}
SGC and SGC-R satisfy the minimal similarity and attribute matching requirements by design, while existing clustering methods cannot accommodate these requirements or cannot scale. We benchmark against two families of methods. The first is \emph{average-metric} methods, which optimize an aggregate objective and do not accept a similarity threshold: Mini-batch K-Means \citep{sculley2010minibatchkmeans}, K-Means \citep{macqueen1967multivariate}, Agglomerative Clustering \citep{Ward1963HierarchicalGT}, BIRCH \citep{zhang1996birch}, Spectral Clustering \citep{shi2000normalized,von2007tutorial}, and Gaussian Mixture \citep{dempster1977gmm}, as implemented in Scikit-learn \citep{scikit-learn}. The second is \emph{threshold-based} methods, most closely related to ours, which enforce a within-cluster similarity threshold: Star Clustering \citep{aslam1998static}, SimClus \citep{alhasan2011}, and Community Detection \citep{reimers-2019-sentence-bert} (Section~\ref{sec:related}, Appendix~\ref{app:related-extended}).

In clustering, there is a general tradeoff between reducing approximation errors and reducing data size, measured by within-cluster similarity and the number of clusters. Increasing the number of clusters generally improves within-cluster similarity at the cost of less data size reduction. The two families require different fair-comparison protocols. Average-metric methods do not accept a similarity threshold, so we specify the \emph{same number of clusters} as our method and compare within-cluster similarity and run time (Table~\ref{tab:sim}). Threshold-based methods, like ours, accept a similarity threshold, so we instead specify the \emph{same threshold} and compare the resulting number of clusters (data reduction), similarity, and run time (Table~\ref{tab:sim-threshold}). We set the minimal similarity threshold between any sample and its cluster representative to 0.75 on Shopping Personas, 0.3 on AG News, and 0.4 on Cosmopedia so that our proposed algorithm attains a $\sim$50-fold reduction in data size across all three datasets ($\alpha$ depends on the dataset/embedding for the same target reduction fold). All methods are run on a single \texttt{r7i.12xlarge} instance without parallelization. 

Table~\ref{tab:sim} shows that against the average-metric methods, our method uniquely guarantees the minimal similarity --- 0\% of samples below the threshold, versus up to 21\% for the baselines --- while finishing in just a few seconds. 
Note that the number of clusters can affect the run time, so using Mini-batch K-Means to generate all the final clusters ($K=2545$) ends
up $\sim$10$\times$ slower than the proposed two-stage approach that uses Mini-batch K-Means to
only generate the initial clusters ($K=40$) before further partition.

Table~\ref{tab:sim-threshold} compares the threshold-based methods, which all attain the guardrail by construction except Community Detection. Three findings stand out. First, the proposed method is \emph{one to three orders of magnitude faster} (e.g., 2.7 seconds vs. 768 seconds for SimClus on Shopping Personas) with correspondingly lower memory (Section~\ref{sec:complexity}), because the baselines materialize the full $O(n^2)$ similarity matrix while ours confines the quadratic computation to within initial clusters. Second, SimClus --- which is a special case of our method with $K=1$ and no attribute matching or reassignment --- produces the fewest clusters (highest data reduction) but at the cost of the largest runtime and memory, and a lower average similarity; our proposed method trades a modest increase in cluster count for its scalability and the highest average similarity. Third, Community Detection does not guarantee the similarity threshold: 36--47\% of its samples fall below $\alpha$, because its overlap-removal step can reassign a sample to a representative outside its threshold (Appendix~\ref{app:threshold-comparison}). Star Clustering, using degree-based rather than coverage-based selection, produces up to $2\times$ more clusters than our method and a lower average similarity than the reassignment variant. 


\begin{table}[t]
  \caption{Comparison against \emph{average-metric} clustering methods on internal Customer Shopping Personas, AG News, and Cosmopedia, at a \emph{matched number of clusters} (equal to the proposed method's). Columns: average within-cluster similarity (Avg. Sim.), minimal within-cluster similarity (Min. Sim.), percent of samples below the desired similarity threshold (Perc. Below Thresh.), and run time in seconds; all on a single \texttt{r7i.12xlarge} without parallelization. Only the proposed method (SGC and SGC-R) guarantees the minimal similarity (0\% below threshold). Table~\ref{tab:sim-threshold} compares threshold-based methods.}
  \label{tab:sim}
  \centering
  \adjustbox{width=\textwidth}{
  \begin{tabular}{lrrrrrrrrrrrr}
    \toprule
    & \multicolumn{4}{c}{\textbf{Shopping Personas}} & \multicolumn{4}{c}{\textbf{AG News}} & \multicolumn{4}{c}{\textbf{Cosmopedia}} \\
    \cmidrule(lr){2-5} \cmidrule(lr){6-9} \cmidrule(lr){10-13}
    Method & \begin{tabular}[c]{@{}r@{}}Avg.\\Sim.\end{tabular} & \begin{tabular}[c]{@{}r@{}}Min.\\Sim.\end{tabular} & \begin{tabular}[c]{@{}r@{}}Perc.\\Below\\Thresh.\end{tabular} & \begin{tabular}[c]{@{}r@{}}Time\\(sec.)\end{tabular}
           & \begin{tabular}[c]{@{}r@{}}Avg.\\Sim.\end{tabular} & \begin{tabular}[c]{@{}r@{}}Min.\\Sim.\end{tabular} & \begin{tabular}[c]{@{}r@{}}Perc.\\Below\\Thresh.\end{tabular} & \begin{tabular}[c]{@{}r@{}}Time\\(sec.)\end{tabular}
           & \begin{tabular}[c]{@{}r@{}}Avg.\\Sim.\end{tabular} & \begin{tabular}[c]{@{}r@{}}Min.\\Sim.\end{tabular} & \begin{tabular}[c]{@{}r@{}}Perc.\\Below\\Thresh.\end{tabular} & \begin{tabular}[c]{@{}r@{}}Time\\(sec.)\end{tabular} \\
    \midrule
    \textbf{SGC}   & 0.802 & \textbf{0.750} & \textbf{0.0\%}  & \textbf{2.7}  & 0.413 & \textbf{0.300}  & \textbf{0.0\%}  & \textbf{4.2}  & 0.512 & \textbf{0.400} & \textbf{0.0\%} & \textbf{6.8} \\
    \textbf{SGC-R} & 0.825 & \textbf{0.750} & \textbf{0.0\%}  & \textbf{2.7} & 0.490 & \textbf{0.300}  & \textbf{0.0\%}  & \textbf{4.2}  & 0.588 & \textbf{0.400} & \textbf{0.0\%} & \textbf{6.8} \\
    Mini-batch K-Means             & 0.819 & 0.554 & 6.1\%  & 44   & 0.553 & 0.006  & 5.7\%  & 33 & 0.619 & -0.005& 4.3\% & 91 \\
    K-Means                        & 0.828 & 0.600 & 2.9\%  & 154  & 0.561 & -0.038 & 4.9\%  & 136  & \textbf{0.630} & 0.133 & 2.9\% & 296 \\
    Agglomerative                  & 0.812 & 0.542 & 10.3\% & 1778 & 0.543 & -0.108 & 9.6\%  & 2854 & 0.608 & 0.007 & 7.0\% & 2851 \\
    BIRCH                          & 0.799 & 0.558 & 14.0\% & 28   & 0.536 & -0.068 & 9.2\%  & 878  & 0.603 & -0.001 & 7.0\% & 1337 \\
    Spectral                       & 0.807 & 0.445 & 11.3\% & 6281 & 0.505 & -0.186 & 21.2\% & 1221 & 0.592 & -0.077 & 13.9\% & 1804 \\
    Gaussian Mixture               & \textbf{0.846} & 0.618 & 0.8\%  & 5548 & \textbf{0.562} & 0.006  & 4.8\%  & 1898 & \textbf{0.630} & 0.090 & 3.0\% & 4254 \\
    \bottomrule
  \end{tabular}
  }
\end{table}

\begin{table}[t]
  \caption{Comparison against \emph{threshold-based} clustering methods at a
  \emph{matched similarity threshold} (0.75 / 0.3 / 0.4), so we compare the
  resulting number of final clusters (\# Clusters; fewer means more data
  reduction); remaining columns and setup as in Table~\ref{tab:sim}. All methods
  guarantee the threshold by construction except Community Detection, which
  reassigns samples during overlap removal. The proposed method attains
  competitive data reduction and the highest average similarity at one-to-three
  orders of magnitude less runtime and memory
  (Section~\ref{sec:complexity}, Appendix~\ref{app:related-extended}).}
  \label{tab:sim-threshold}
  \centering
  \adjustbox{width=\textwidth}{
  \begin{tabular}{lrrrrrrrrrrrr}
    \toprule
    & \multicolumn{4}{c}{\textbf{Shopping Personas}} & \multicolumn{4}{c}{\textbf{AG News}} & \multicolumn{4}{c}{\textbf{Cosmopedia}} \\
    \cmidrule(lr){2-5} \cmidrule(lr){6-9} \cmidrule(lr){10-13}
    Method & \begin{tabular}[c]{@{}r@{}}\#\\Clusters\end{tabular} & \begin{tabular}[c]{@{}r@{}}Avg.\\Sim.\end{tabular} & \begin{tabular}[c]{@{}r@{}}Perc.\\Below\\Thresh.\end{tabular} & \begin{tabular}[c]{@{}r@{}}Time\\(sec.)\end{tabular}
           & \begin{tabular}[c]{@{}r@{}}\#\\Clusters\end{tabular} & \begin{tabular}[c]{@{}r@{}}Avg.\\Sim.\end{tabular} & \begin{tabular}[c]{@{}r@{}}Perc.\\Below\\Thresh.\end{tabular} & \begin{tabular}[c]{@{}r@{}}Time\\(sec.)\end{tabular}
           & \begin{tabular}[c]{@{}r@{}}\#\\Clusters\end{tabular} & \begin{tabular}[c]{@{}r@{}}Avg.\\Sim.\end{tabular} & \begin{tabular}[c]{@{}r@{}}Perc.\\Below\\Thresh.\end{tabular} & \begin{tabular}[c]{@{}r@{}}Time\\(sec.)\end{tabular} \\
    \midrule
    \textbf{SGC}   & 2545 & 0.802 & \textbf{0.0\%} & \textbf{2.7} & 1795 & 0.413 & \textbf{0.0\%} & \textbf{4.2} & 2399 & 0.512 & \textbf{0.0\%} & \textbf{6.8} \\
    \textbf{SGC-R} & 2545 & \textbf{0.825} & \textbf{0.0\%} & \textbf{2.7} & 1795 & \textbf{0.490} & \textbf{0.0\%} & \textbf{4.2} & 2399 & \textbf{0.588} & \textbf{0.0\%} & \textbf{6.8} \\
    SimClus                        & \textbf{1466} & 0.783 & \textbf{0.0\%} & 768 & \textbf{702} & 0.369 & \textbf{0.0\%} & 1892 & \textbf{1241} & 0.476 & \textbf{0.0\%} & 1822 \\
    Star (partition)               & 5173 & 0.811 & \textbf{0.0\%} & 26 & 1675 & 0.470 & \textbf{0.0\%} & 22 & 2903 & 0.562 & \textbf{0.0\%} & 21 \\
    Star (overlap)                 & 5173 & 0.767 & \textbf{0.0\%} & 25 & 1675 & 0.355 & \textbf{0.0\%} & 22 & 2903 & 0.455 & \textbf{0.0\%} & 21 \\
    Community Det.            & 4563 & 0.761 & 35.7\% & 71 & 7485 & 0.337 & 47.3\% & 28 & 9009 & 0.449 & 45.9\% & 24 \\
    \bottomrule
  \end{tabular}
  }
\end{table}

\subsection{Cluster Size Distributions}
\label{sec:cluster_size}
As detailed in Section \ref{sec:method}, the iterative greedy selection of representative samples in the proposed algorithm by design leads to a highly skewed cluster size distribution, with practical advantages: we can remove a large percentage of tail clusters at the cost of removing a small percentage of samples, improving data size reduction. Additionally, if quality assurance is expensive (e.g., human-in-the-loop), we can focus on validating the top clusters by size. We examine how many overall samples the top clusters by size cover for different clustering methods in Figures~\ref{fig:size} and~\ref{fig:size2}, based on the experiment output from Section \ref{sec:sim}.

We see that as the percentage of top clusters increases, the proposed algorithm covers significantly more samples than the average-metric baselines, which produce more balanced clusters. Among the threshold-based methods, the greedy coverage-based methods (SGC and SimClus) produce the most skewed distributions and thus the most trimming-friendly coverage, while the degree- and neighborhood-based methods (Star Clustering, Community Detection) are less skewed. The practical implication is that with SGC we can keep only $\sim$4\% of clusters to cover 90\% of Shopping Personas, offering an additional $\sim$25-fold reduction in downstream computations at the cost of removing 10\% of samples, which can be compensated for by oversampling in the beginning.

\begin{figure}[t!]
  \centering
  \includegraphics[width=1.0\textwidth]{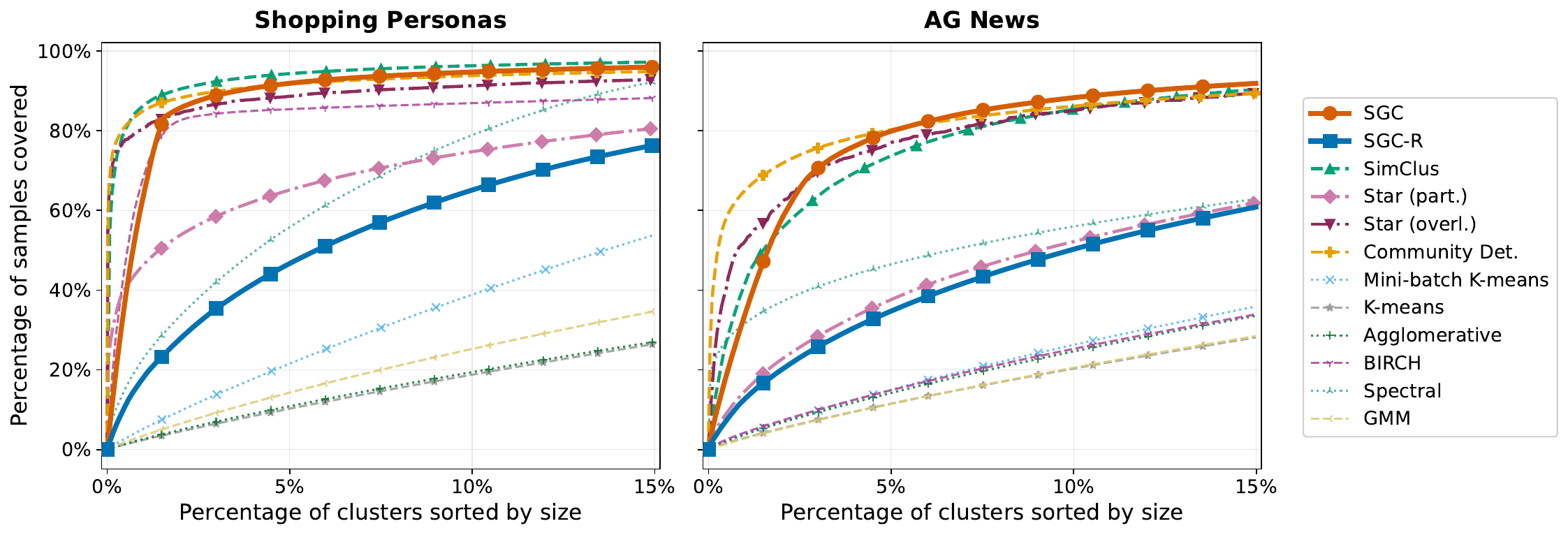}
  \caption{Sample coverage by sorted clusters on Shopping Personas and AG News. Greedy coverage-based methods (SGC, SimClus) yield the most skewed, trimming-friendly distributions.}
  \label{fig:size}
\end{figure}

\subsection{Impact of Within-Cluster Similarity on Recommendation Relevance}
\label{sec:rel}
Recall that for our application in eCommerce, after clustering, a customer will inherit the recommendations based on their cluster representative's shopping persona. A key intuition behind using the minimal within-cluster similarity as the quality guardrails is: The more similar a customer is to their cluster representative, the more relevant the recommendations based on the representative will be to that customer. In this experiment we empirically validate the impact of within-cluster similarity on the relevance of recommendations.

We randomly sample 200 customers as the cluster representatives that the product recommendations will be based on. For each representative, we perform a stratified sampling of its cluster members based on the persona similarity between the member and the representative: We sample 5 customers for each of the 7 similarity buckets (0.2 - 0.3, 0.3 - 0.4, $\ldots$, 0.8 - 0.9). This sampling process yields 7,000 member-representative pairs of varying similarities. We generate 15 product recommendations based on the shopping persona of each cluster representative, let the cluster members of each representative inherit the same recommendations, then use the LLM-based Marketing Critic to judge the relevance between each product recommendation and the customer's shopping persona on a scale of 1-5 from least relevant to perfectly relevant. We compute the linear correlation and Spearman Rank Correlation between the member-representative similarity and the average product relevance rating of the 7,000 member-representative pairs; and find the correlations to be 0.781 and 0.797 respectively (p-value < 0.001 for both metrics). The scatterplot of the average relevance rating versus the member-representative similarity is shown in Figure \ref{fig:rel}. The significant correlations help validate using the within-cluster similarity as quality guardrails of product recommendations.

\begin{figure}[t]
  \centering
  \includegraphics[width=0.7\columnwidth]{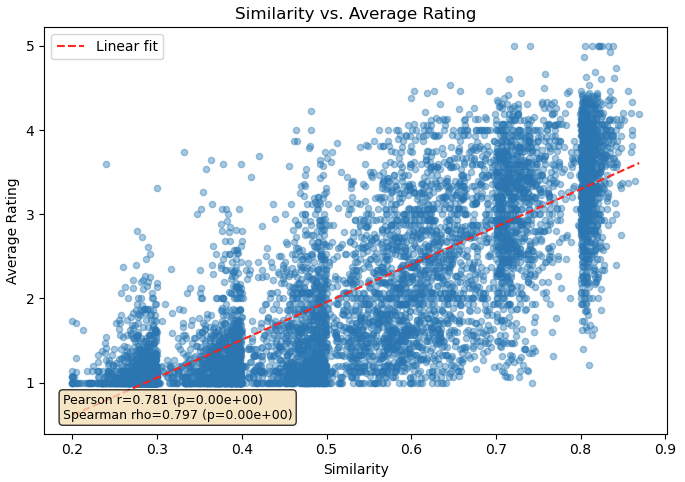}
  \caption{Recommendation relevance rating by member-representative similarity.}
  \label{fig:rel}
\end{figure}

\subsection{Personalization vs.\ Data Reduction}
\label{sec:tradeoff}
The user-specified minimal similarity $\alpha$ and matching-attribute
set jointly control the tradeoff between personalization and data
reduction. On a 5M-customer subset, with no attribute matching, the
proposed algorithm achieves $186\times$ reduction at $\alpha=0.75$ and
$998\times$ at $\alpha=0.7$; adding exact matching of adult gender,
adult count, child presence, child gender, and child age tightens
these to $51\times$ and $100\times$ respectively, while keeping
the average within-cluster similarity above $0.79$. The change in data reduction when household attributes are added also indicates that persona similarity alone is insufficient to ensure that these key attributes match within a final cluster, hence the necessity of attribute matching.
Practitioners can tune both the minimal similarity and the matching attributes to navigate the tradeoff between personalization and data reduction:
stronger guardrails help preserve personalization but reduce the fold of data reduction, while never violating the guarantees of Theorem~\ref{thm:guardrail}.
In practice the data reduction fold is often determined by budget constraints, and the guardrails are set based on safety/quality constraints and the target reduction fold.
For a new dataset, these values can be chosen by sweeping $\alpha$ and the optional matching attributes on a small random sample to hit the target.
The full table sweeping across thresholds and
attribute configurations on the 5M-customer subset is provided in Appendix~\ref{app:dedup_fold}.

\section{Application}
\label{sec:app}
We tested and launched SGC in an A/B test in June 2025, validating the end-to-end recommender system (Figure \ref{fig:crm}, introduced in Section~\ref{sec:intro}) that generates personalized product recommendations from shopping personas and household attributes estimated from interaction histories. Each component in the pipeline is separately developed and validated; in particular, the Marketing Critic achieves 97\% precision and 60\% recall in persona-item relevance predictions on 6,000 human-annotated test samples. Because the pipeline is almost entirely LLM-based, it faces scalability issues at the targeted scale of 38 million customers. Our clustering algorithm makes the A/B test feasible by reducing the data size for downstream computations by $\sim$50 times, while guaranteeing a minimal customer persona similarity of 0.77 and exact matching of adult gender, adult count, child presence, child gender, and child age within each cluster. Thus, our clustering algorithm plays a key role in unblocking and supporting the launch.

Concretely, without clustering the query generation step requires \$114,000 and 22.8 days (Haiku 3 on demand, limited by throughput) and the Marketing Critic step \$1,018,500 and 485 days (Nova Micro on demand); with clustering, these drop to \$2,100 and 0.42 days, and \$20,370 and 9.7 days, respectively. Further reduction can be achieved as needed by removing tail clusters (Section~\ref{sec:cluster_size}) or lowering the required minimal similarity (Section~\ref{sec:tradeoff}).

\begin{sloppypar}
We also evaluate the impact of clustering on recommendation quality. Ideally, cluster members receive recommendations as relevant to them as to their representative, i.e., the product-to-representative-customer relevance rate should be close to the product-to-cluster-member relevance rate for recommendations personalized to the representative. Applying the Marketing Critic to 5,000 randomly sampled representatives and their corresponding cluster members, we find that the relative difference between the two rates is only 0.7\% (95\% Confidence Interval $[-0.1\%, 1.5\%]$).
%
%
The small difference helps validate that the proposed clustering method is able to retain personalization while increasing scalability in the case of personalized recommendations.
\end{sloppypar}

The A/B test showed significant positive impact on revenues and engagement
(p-value < 0.05)
and enabled subsequent launches.

\section{Conclusion}
\label{sec:disc}
In this work, we propose SGC, a scalable two-stage clustering algorithm with guarantees of minimal similarity and attribute matching between each sample and its cluster representative as quality guardrails while reducing downstream cost and latency. We provide empirical evidence that the proposed method uniquely combines per-sample guardrails with scalability, running one to three orders of magnitude faster than benchmark methods, and that its skewed cluster sizes enable trimming tail clusters for further data reduction. We apply SGC to a shopping-persona-based recommender system for 38 million customers, cutting downstream LLM computation by ${\sim}50\times$ (from over \$1.1M and 500 days to ${\sim}$\$22K and ${\sim}$10 days) and unblocking the launch. For future work, richer ways to measure inter-customer similarity and select cluster representatives could tighten the link between the guardrails and downstream recommendation relevance.

\begin{acknowledgments}
The paper focuses on the novel clustering method and its application, while the persona-based recommender system (Section~\ref{sec:app}) involves many contributors and multi-team collaborations, including but not limited to our engineering partners Nicholas Rahardja, Sri Venkata Vivek Dhulipala, Minh Le, and Daniel Di Pasquale.
\end{acknowledgments}

\section*{Declaration on Generative AI}
During the preparation of this work, the author(s) used Claude Opus 4.7 and 4.8 to identify and correct grammatical errors, typos, and other writing mistakes of the initial draft and make format improvements. No AI tools were used for ideation, data analysis, experimentation, the formulation of conclusions, or the initial drafting. After using these tool(s)/service(s), the author(s) reviewed and edited the content as needed and take(s) full responsibility for the publication's content.

\bibliography{amlc_2025_workshop}

\appendix

\section{Extended Related Work}
\label{app:related-extended}

This appendix expands on Section \ref{sec:related}: we first give a detailed
comparison against the threshold-based clustering methods most related to ours,
then discuss other LLM-scaling approaches and argue that they are complementary
to --- rather than substitutes for --- input-side clustering for our use case.

\subsection{Detailed Comparison with Threshold-Based Clustering}
\label{app:threshold-comparison}

All threshold-based methods we consider impose a user-specified lower bound
$\alpha$ on within-cluster similarity, but they differ in how they select
representatives, whether they guarantee the bound, and whether they scale.
Table \ref{tab:related-comparison} summarizes the comparison.

\paragraph{Star Clustering \citep{aslam1998static}.}
Star Clustering covers a thresholded similarity graph with star-shaped
subgraphs: it repeatedly selects the highest-degree uncovered vertex as a star
center and assigns its neighbors as satellites. The threshold is guaranteed only
between each satellite and its center, not between arbitrary pairs. Because
degree-based selection produces an \emph{independent dominating set} of centers,
it selects more representatives than necessary --- its approximation ratio on the
number of clusters can be as poor as $\Omega(n)$~\citep{alhasan2011} --- which
translates to weaker data reduction in our experiments (Section \ref{sec:sim}).
It also requires the full $O(n^2)$ similarity matrix. A ``partitioning'' variant
reassigns each member to its most-similar center to produce disjoint clusters,
analogous to our reassignment step but built on degree-based rather than
coverage-based representatives.

\paragraph{SimClus \citep{alhasan2011}.}
SimClus reformulates lower-bound-similarity clustering as \textsc{Set Cover} and
applies the Johnson--Chv\'atal greedy heuristic
\citep{johnson1974setcover,chvatal1979setcover}: at each step it selects the
sample covering the most still-uncovered samples, achieving an $O(\log n)$
approximation on the number of clusters. This coverage-based selection is exactly
our Stage~2, so SimClus produces the fewest clusters (best data reduction) among
the baselines. However, it computes the entire $O(n^2)$ similarity matrix and
selects globally, making it the slowest and most memory-intensive method in
practice (Section \ref{sec:sim}, Section \ref{app:comp-comp}). Our method can be
seen as SimClus made scalable by confining the greedy selection to within
Mini-batch K-Means initial clusters, augmented with attribute matching and an
optional reassignment step; the price is a modest increase in the number of
clusters, since two similar samples split across different initial clusters
cannot share a representative (Remark on partition-constrained vs.\ global
optimum in Appendix \ref{app:formal-guarantees}).

\paragraph{Community Detection \citep{reimers-2019-sentence-bert}.}
The Community Detection utility in SentenceTransformers forms, for each point, a
community of all points above the similarity threshold, sorts communities by
size, and greedily removes overlap. Two caveats matter for our use case. First,
during overlap removal a member's original seed can be claimed by a larger
community, so the member is reassigned to a new community whose reported
representative may be \emph{below} the threshold --- hence Community Detection does
not strictly guarantee the minimal similarity (we observe $36$--$47\%$ of samples
below $\alpha$ in Section \ref{sec:sim}). Second, its default drops points in
communities smaller than a minimum size; we set this minimum to $1$ for a
full-coverage comparison. To scale, SentenceTransformers batches the data and runs
Community Detection within each batch --- this can be viewed as a special case of
performing an initial round of clustering with a cheap method, except that random
batching is more likely than Mini-batch K-Means to separate similar samples into
different batches, yielding more final clusters.

\paragraph{SemDedup \citep{abbas2023semdedupdataefficientlearningwebscale}.}
SemDedup shares our two-stage structure (cluster embeddings, then compare within
each cluster) and was designed for web-scale data. However, its objective is to
\emph{remove} semantic duplicates to shrink a training set: within each cluster it
keeps one exemplar per group of near-duplicates and discards the rest, retaining
no mapping from a discarded sample to its exemplar. Repurposing it as a clustering
method for output propagation would require inventing such an assignment, and its
similarity threshold governs the retained set rather than a per-sample
member-to-representative guardrail. We therefore address it in discussion but do
not benchmark it; SimClus is the faithful, clustering-native representative of the
coverage-based threshold family that we do benchmark.

\begin{table}[t]
  \caption{Comparison of threshold-based clustering methods. ``Selection'' is the
  representative-selection rule; ``Guarantees $\alpha$'' is whether every member
  is guaranteed similarity $\geq \alpha$ to its representative; ``Attr.\ match''
  is exact attribute matching; ``Scalable'' is whether the method avoids
  materializing the full $O(n^2)$ similarity matrix over all samples. SemDedup
  ($\ast$) bounds similarity only within the retained set, not per-sample to a
  representative, and retains no member-to-representative mapping.}
  \label{tab:related-comparison}
  \centering
  \adjustbox{width=0.7\columnwidth}{
  \begin{tabular}{lcccc}
    \toprule
    Method & Selection & \begin{tabular}[c]{@{}c@{}}Guar.\\$\alpha$\end{tabular} & \begin{tabular}[c]{@{}c@{}}Attr.\\match\end{tabular} & Scalable \\
    \midrule
    Star Clustering & degree & \checkmark & \ding{55} & \ding{55} \\
    SimClus & coverage & \checkmark & \ding{55} & \ding{55} \\
    Community Det. & neighborhood & \ding{55} & \ding{55} & \ding{55} \\
    SemDedup & farthest-centroid & \ding{55}$^\ast$ & \ding{55} & \checkmark \\
    \midrule
    \textbf{SGC} & coverage & \checkmark & \checkmark & \checkmark \\
    \bottomrule
  \end{tabular}
  }
\end{table}

\subsection{Complementary LLM-Scaling Approaches}
\label{app:llm-scaling}

Input-side clustering is one of several strategies for reducing LLM inference
cost, and it is complementary to the others rather than mutually exclusive.
\emph{KV caching} and prefix reuse \citep{Kwon2023,zheng2024sglang} amortize the
cost of a shared prompt prefix across requests, and \emph{prompt caching} is now
offered by major providers. \emph{Model compression} --- distillation
\citep{hinton2015distilling}, quantization \citep{dettmers2022llmint8}, and
pruning \citep{frantar2023sparsegpt} --- reduces the per-call cost of the model
itself. \emph{Inference-serving} optimizations such as continuous batching and
paged attention \citep{Kwon2023,zhu2025nanoflowoptimallargelanguage} raise
throughput. These operate on the prompt, the model, or the serving stack; our
approach instead reduces the \emph{number of distinct inputs} the LLM must
process, and can be layered on top of any of them.

Crucially, none of these alone addresses our target use case. KV/prompt caching
helps only when requests share a long prefix; our millions of shopping personas
are semantically similar but do not share verbatim prefixes, so caching yields
little reuse. Model compression reduces per-call cost but not the number of
calls, and is itself a non-trivial, time-consuming effort: distillation requires
running the expensive teacher over large data to generate training targets, then
training and evaluating student models, with quality risk on the exact
personalization and safety behaviors we must preserve. Clustering, by contrast,
reduces the call count by $\sim$$50\times$ with a provable per-sample guardrail
and negligible overhead, and remains compatible with caching, compression, and
serving optimizations applied to the reduced set of representative calls.

\section{Formal Guarantees of the Proposed Algorithm}
\label{app:formal-guarantees}

In this section we prove that both Algorithm~\ref{algo1} (the base
two-stage algorithm) and Algorithm~\ref{algo2} (the optional
reassignment variant) satisfy the minimal similarity and attribute
matching requirements stated in Section~\ref{sec:intro} (Requirements~1
and~2). We formalize the argument through the lens of \emph{set cover},
which naturally captures the representative-selection step.

\subsection{Preliminaries and Notation}

Let $\mathcal{D} = \{(P_i, A_i)\}_{i=1}^n$ be the input dataset, with
embeddings $E_i = f_{\text{emb}}(P_i)$. For a user-specified similarity
threshold $\alpha \in [-1, 1]$ and attribute space, define the
\emph{match relation} $\sim_\alpha$ on $\mathcal{D}$ by
$$
i \sim_\alpha j \;\iff\; f_{\text{sim}}(E_i, E_j) \ge \alpha
\;\text{ and }\; A_i = A_j.
$$
For any $i \in \{1,\dots,n\}$, let the \emph{$\alpha$-ball} (or
\emph{admissible set}) of $i$ be
$$
\mathcal{B}_\alpha(i) \;=\; \{\, j \in \{1,\dots,n\} : i \sim_\alpha j \,\}.
$$
Note that $\sim_\alpha$ is reflexive (since $f_{\text{sim}}(E_i, E_i)=1\ge\alpha$
for any reasonable $\alpha\le 1$, and trivially $A_i = A_i$) and
symmetric (since $f_{\text{sim}}$ is symmetric and equality of attributes
is symmetric), so $i \in \mathcal{B}_\alpha(i)$ and
$j \in \mathcal{B}_\alpha(i) \iff i \in \mathcal{B}_\alpha(j)$. It is
\emph{not} transitive in general, which is why clustering under this
relation is non-trivial.

A cluster assignment $\hat{f}_{\text{cluster}} : \mathcal{D} \to \{1,\dots,n\}$
with representative set $\mathcal{R} \subseteq \{1,\dots,n\}$ is said
to satisfy the \emph{guardrail property} if
\begin{equation}
\hat{f}_{\text{cluster}}(P_i, A_i) = r
\;\Longrightarrow\;
r \in \mathcal{R} \;\text{ and }\; i \in \mathcal{B}_\alpha(r).
\label{eq:guardrail}
\end{equation}
Equivalently, \eqref{eq:guardrail} says that every customer is assigned
to a representative whose $\alpha$-ball \emph{covers} them. This is
precisely the condition required by Requirements~1 and~2.

\subsection{Set-Cover Interpretation}

Within any initial cluster $C_k \subseteq \{1,\dots,n\}$ produced in
Stage~1, consider the collection of candidate cover sets
$$
\mathcal{F}_k \;=\; \bigl\{\, \mathcal{B}_\alpha(i) \cap C_k
\;:\; i \in C_k \,\bigr\}.
$$
Stage~2 of Algorithm~\ref{algo1} is exactly the classical greedy
algorithm for \textsc{Set Cover} applied to the universe $C_k$ with
candidate sets $\mathcal{F}_k$: at each iteration it picks the set
covering the largest number of still-uncovered elements, removes those
elements, and repeats until the universe is exhausted. Because
$\sim_\alpha$ is reflexive, every element $i \in C_k$ is contained in
at least one candidate set (namely $\mathcal{B}_\alpha(i) \cap C_k$),
so a valid cover always exists and the greedy procedure terminates.

\subsection{Correctness of Algorithm~\ref{algo1}}

\begin{theorem}[Guardrail guarantee for Algorithm~\ref{algo1}]
\label{thm:alg1}
For any input $\mathcal{D}$, any embedding function $f_{\text{emb}}$,
any symmetric similarity function $f_{\text{sim}}$ with
$f_{\text{sim}}(x,x) \ge \alpha$, any threshold $\alpha$, and any number
of initial clusters $K \ge 1$, the assignment $\hat{f}_{\text{cluster}}$
returned by Algorithm~\ref{algo1} satisfies the guardrail property
\eqref{eq:guardrail}. In particular, for every $i \in \{1,\dots,n\}$,
letting $r = \hat{f}_{\text{cluster}}(P_i, A_i)$,
$$
f_{\text{sim}}\!\bigl(f_{\text{emb}}(P_i), f_{\text{emb}}(P_r)\bigr)
\;\ge\; \alpha
\qquad\text{and}\qquad
A_i = A_r.
$$
\end{theorem}

\begin{proof}
Fix an initial cluster $C_k$ produced in Stage~1, and consider Stage~2
applied to $C_k$. Let $\mathcal{U}^{(t)}$ denote the set of unmatched
customers at the start of iteration $t$, with $\mathcal{U}^{(1)} = C_k$.

At iteration $t$, the algorithm computes, for each $i \in C_k$,
$$
\mathcal{M}_i^{(t)} \;=\; \{\, j \in \mathcal{U}^{(t)} \,:\, M_{ij} = 1 \,\}
\;=\; \mathcal{B}_\alpha(i) \cap \mathcal{U}^{(t)},
$$
where the second equality follows from the definition
$M_{ij} = \mathbf{1}\{S_{ij} \ge \alpha \text{ and } A_i = A_j\}$ in
line~7 of Algorithm~\ref{algo1}. It then selects
$r^{*(t)} = \arg\max_{i \in C_k} |\mathcal{M}_i^{(t)}|$ and assigns
every $j \in \mathcal{M}_{r^{*(t)}}^{(t)}$ to representative $r^{*(t)}$
(line~14). By construction,
$$
j \in \mathcal{M}_{r^{*(t)}}^{(t)}
\;\Longrightarrow\;
j \in \mathcal{B}_\alpha(r^{*(t)}),
$$
so $f_{\text{sim}}(E_j, E_{r^{*(t)}}) \ge \alpha$ and
$A_j = A_{r^{*(t)}}$. The set $\mathcal{M}_{r^{*(t)}}^{(t)}$ is then
removed from $\mathcal{U}^{(t+1)}$ (line~15), so no customer is ever
reassigned within Stage~2, and every customer assigned in iteration $t$
satisfies \eqref{eq:guardrail}.

Since $\sim_\alpha$ is reflexive, $r^{*(t)} \in \mathcal{M}_{r^{*(t)}}^{(t)}$,
so at least one element is removed per iteration, guaranteeing
termination in at most $|C_k|$ iterations with
$\mathcal{U}^{(T+1)} = \emptyset$. Hence every $i \in C_k$ is assigned
to some representative $r \in C_k$ with $i \in \mathcal{B}_\alpha(r)$.
Because Stage~1 partitions $\{1,\dots,n\}$ into $\{C_k\}_{k=1}^{K}$ and
Stage~2 is applied independently within each $C_k$, the guarantee
extends to every $i \in \{1,\dots,n\}$.
\end{proof}

\begin{corollary}[Cover interpretation]
Let $\mathcal{R}_k \subseteq C_k$ be the representatives selected by
Stage~2 within initial cluster $C_k$. Then $\{\mathcal{B}_\alpha(r)
\cap C_k\}_{r \in \mathcal{R}_k}$ is a valid cover of $C_k$, and
$\mathcal{R} = \bigcup_{k=1}^{K} \mathcal{R}_k$ induces the final
cluster assignment. The final number of clusters is exactly
$|\mathcal{R}| = \sum_{k=1}^{K} |\mathcal{R}_k|$.
\end{corollary}

\begin{remark}[Role of the initial clustering]
Theorem~\ref{thm:alg1} holds for \emph{any} partition
$\{C_k\}_{k=1}^{K}$ used in Stage~1, including the trivial partition
$K=1$. The choice of initial clustering affects only the number of
final clusters and the runtime/memory,
not the correctness of the guardrails. This is consistent with the
empirical robustness reported in Table \ref{tab:robustness}.
\end{remark}

\subsection{Correctness of Algorithm~\ref{algo2} (Reassignment Variant)}

\begin{theorem}[Guardrail guarantee for Algorithm~\ref{algo2}]
\label{thm:alg2}
Under the assumptions of Theorem~\ref{thm:alg1}, the assignment
$\hat{f}_{\text{cluster}}$ returned after applying
Algorithm~\ref{algo2} to the output of Algorithm~\ref{algo1}
also satisfies the guardrail property \eqref{eq:guardrail}. Moreover,
for every customer $i$, the post-reassignment similarity is at least as
large as the original:
$$
\fsim\!\left(E_i,\, E_{\hat{f}_{\mathrm{cluster}}^{\mathrm{new}}(P_i, A_i)}\right)
\;\ge\;
\fsim\!\left(E_i,\, E_{\hat{f}_{\mathrm{cluster}}^{\mathrm{old}}(P_i, A_i)}\right).
$$
where $\hat{f}_{\text{cluster}}^{\mathrm{old}}$ and
$\hat{f}_{\text{cluster}}^{\mathrm{new}}$ denote the assignments before
and after reassignment, respectively.
\end{theorem}

\begin{proof}
Fix an initial cluster $C_k$ and let
$\mathcal{R}_k \subseteq C_k$ be the representatives produced by
Algorithm~\ref{algo1}. For each $i \in C_k$,
Algorithm~\ref{algo2} (line~2) defines
$$
\begin{aligned}
r^* \;&=\; \mathrm{argmax}_{r \in \mathcal{R}_k \,:\, M_{ir} = 1} S_{ir} \\
    \;&=\; \mathrm{argmax}_{r \in \mathcal{R}_k \,\cap\, \mathcal{B}_\alpha(i)} \fsim(E_i, E_r).
\end{aligned}
$$
The reassignment is performed only if the constraint set
$\mathcal{R}_k \cap \mathcal{B}_\alpha(i)$ is nonempty (line~3), in
which case $r^* \in \mathcal{B}_\alpha(i)$ by definition, so the
guardrail property \eqref{eq:guardrail} is preserved.

If $\mathcal{R}_k \cap \mathcal{B}_\alpha(i)$ is empty, the assignment
is left unchanged; by Theorem~\ref{thm:alg1} the original assignment
already satisfied the guardrail, so it still does.

However, by Theorem~\ref{thm:alg1} the original representative
$r^{\mathrm{old}} = \hat{f}_{\text{cluster}}^{\mathrm{old}}(P_i, A_i)$
belongs to $\mathcal{R}_k$ and satisfies $i \in \mathcal{B}_\alpha(r^{\mathrm{old}})$,
equivalently $r^{\mathrm{old}} \in \mathcal{B}_\alpha(i)$ by symmetry
of $\sim_\alpha$. Hence $\mathcal{R}_k \cap \mathcal{B}_\alpha(i)$ is
\emph{always} nonempty, and the reassignment always executes. Moreover,
since $r^{\mathrm{old}}$ is a feasible candidate in the $\arg\max$,
$$
f_{\text{sim}}(E_i, E_{r^*}) \;\ge\; f_{\text{sim}}(E_i, E_{r^{\mathrm{old}}}),
$$
establishing the monotonicity claim.
\end{proof}

\begin{remark}[Representative set is preserved]
Algorithm~\ref{algo2} does not modify $\mathcal{R}_k$; it only
redistributes the non-representative customers within $C_k$ among the
existing representatives. Therefore the final number of clusters
$|\mathcal{R}|$ is identical for both variants, while the within-cluster
similarities can only weakly increase. This explains the empirical
observation in Table~\ref{tab:sim} that SGC-R
achieves higher average similarity at the same minimal similarity
$\alpha$ and the same number of clusters as SGC.
\end{remark}

\begin{remark}[Trade-off with tail-cluster trimming]
While reassignment preserves the guardrail and improves average
similarity, it redistributes mass from the largest (earliest-selected)
clusters to smaller ones, flattening the cluster-size distribution.
Tail-cluster trimming (Section~\ref{sec:cluster_size}) therefore covers fewer customers
per retained cluster under Algorithm~\ref{algo2}, consistent
with Figures~\ref{fig:size} and \ref{fig:size2}.
\end{remark}

\subsection{Relation to the Minimum Set Cover Problem}

The greedy step in Stage~2 is the classical
Johnson--Chv\'{a}tal greedy heuristic for \textsc{Set Cover}
\citep{johnson1974setcover,chvatal1979setcover}. Let
$\mathrm{OPT}_k$ denote the minimum number of $\alpha$-balls needed to
cover $C_k$. Then the greedy procedure yields
$$
|\mathcal{R}_k| \;\le\; \mathrm{OPT}_k \cdot \bigl(1 + \ln |C_k|\bigr),
$$
so the total number of final clusters satisfies
$$
|\mathcal{R}| \;=\; \sum_{k=1}^{K} |\mathcal{R}_k|
\;\le\; \bigl(1 + \ln \max_k |C_k|\bigr) \sum_{k=1}^{K} \mathrm{OPT}_k.
$$
This provides a formal ceiling on the data reduction loss incurred by
the greedy representative selection, relative to the (NP-hard)
optimum cover within each initial cluster. We emphasize that
this bound concerns data-reduction efficiency, not correctness:
the guardrail property \eqref{eq:guardrail} holds exactly, not
approximately.

\begin{remark}[Partition-constrained vs.\ global optimum]
The bound above compares $|\mathcal{R}|$ to the optimum
$\mathcal{R}^*_{\text{partition}}$ subject to the Stage~1 partition
$\{C_k\}$. The unconstrained optimum $\mathcal{R}^*_{\text{global}}$
(ignoring partition boundaries) can be strictly smaller, since two
similar customers split across different $C_k$ cannot share a
representative in our algorithm. The gap $|\mathcal{R}^*_{\text{partition}}|
- |\mathcal{R}^*_{\text{global}}|$ is controlled by the quality of the
initial clustering; in our setting Mini-batch K-Means is empirically
sufficient to keep this gap small (Table \ref{tab:robustness} in Appendix \ref{app:robust}).
\end{remark}

\section{Computational Complexity Analysis}
\label{app:complexity}

We analyze the time and memory complexity of
Algorithm~\ref{algo1} and Algorithm~\ref{algo2} as a function
of the dataset size $n$, the embedding dimension $d$, the number of
initial clusters $K$, and the attribute-vector dimension $a$. Let
$n_k = |C_k|$ denote the size of the $k$-th initial cluster, with
$\sum_{k=1}^{K} n_k = n$.

\subsection{Stage~1: Initial Clustering with Mini-Batch K-Means}

Let $b$ be the Mini-batch K-Means batch size and $T_1$ the number of
iterations. A standard Mini-batch K-Means update evaluates $b$ points
against $K$ centroids in $\mathbb{R}^d$ and performs $b$ centroid
updates per iteration, yielding
$$
T_{\text{Stage 1}} \;=\; O\bigl(T_1 \, b \, K \, d\bigr),
\qquad
M_{\text{Stage 1}} \;=\; O\bigl((n + K)\, d\bigr),
$$
where the memory term accounts for storing the $n$ embeddings and $K$
centroids. Since $b$ and $T_1$ are user-specified constants independent
of $n$, Stage~1 is effectively linear in $n$: $O(n \, d)$ memory and
(amortized) $O(K \, d)$ time per mini-batch update. This is consistent
with \citet{sculley2010minibatchkmeans}.

\subsection{Stage~2: Pairwise Similarity, Match Matrix, and Greedy Selection}

Stage~2 is executed independently within each initial cluster $C_k$.
Within $C_k$ we perform the following substeps.

\paragraph{(a) Pairwise similarity matrix.}
Computing $S_{ij} = f_{\text{sim}}(E_i, E_j)$ for all
$i, j \in C_k$ costs $O(n_k^2 \, d)$ time and $O(n_k^2)$ memory.

\paragraph{(b) Attribute match matrix.}
Computing $M_{ij} = \mathbf{1}\{S_{ij} \ge \alpha \text{ and } A_i = A_j\}$
costs $O(n_k^2 \, a)$ time and $O(n_k^2)$ memory (assuming $a$-dimensional
categorical attributes compared in $O(a)$). In practice one can
precompute attribute group IDs in $O(n_k \, a)$ and reduce the pairwise
attribute check to an $O(1)$ equality test, giving $O(n_k^2)$ total.

\paragraph{(c) Iterative greedy representative selection.}
Let $T_k$ be the number of greedy iterations within $C_k$ (i.e.,
$T_k = |\mathcal{R}_k|$). In each iteration the algorithm must find
$
r^* = \arg\max_{i \in C_k} |\mathcal{M}_i^{(t)}|,
$
where $\mathcal{M}_i^{(t)} = \{j \in \mathcal{U}^{(t)} : M_{ij}=1\}$.
A straightforward implementation maintains a row-sum vector of $M$
restricted to $\mathcal{U}^{(t)}$:
\begin{itemize}
  \item Initialization of row sums: $O(n_k^2)$.
  \item Per iteration: select $r^*$ in $O(n_k)$, then update row sums
  by subtracting the columns of $M$ indexed by
  $\mathcal{M}_{r^*}^{(t)}$, costing
  $O(|\mathcal{M}_{r^*}^{(t)}| \cdot n_k)$.
\end{itemize}
Since the sets $\mathcal{M}_{r^{*(t)}}^{(t)}$ partition $C_k$,
$\sum_t |\mathcal{M}_{r^{*(t)}}^{(t)}| = n_k$, and the total cost of
all updates across all $T_k$ iterations is $O(n_k^2)$. Hence the
greedy loop runs in $O(n_k^2)$ time.

\paragraph{Per-cluster Stage 2 complexity.}
Combining (a)--(c):
\begin{align*}
T_{\text{Stage 2}}(C_k) &\;=\; O\bigl(n_k^2 \, d + n_k^2\bigr)
\;=\; O(n_k^2 \, d), \\
M_{\text{Stage 2}}(C_k) &\;=\; O(n_k^2).
\end{align*}

\paragraph{Aggregate Stage~2 complexity.}
Summing over initial clusters,
$$
T_{\text{Stage 2}} \;=\; O\!\left(d \sum_{k=1}^{K} n_k^2\right),
\quad
M_{\text{Stage 2}} \;=\; O\!\left(\max_{k} n_k^2\right),
$$
assuming Stage~2 is processed one initial cluster at a time (so that
only one pairwise matrix is held in memory). If initial clusters are
roughly balanced with $n_k \approx n/K$,
$$
T_{\text{Stage 2}} \;=\; O\!\left(\tfrac{n^2 d}{K}\right),
\qquad
M_{\text{Stage 2}} \;=\; O\!\left(\tfrac{n^2}{K^2}\right).
$$
This formalizes the role of $K$: increasing $K$ decreases Stage~2 time
\emph{linearly} and peak memory \emph{quadratically}, at the cost of
slightly reduced data-reduction efficiency (Appendix \ref{app:robust}, Table \ref{tab:robustness}).

\subsection{Stage~2b: Reassignment Step (Algorithm~\ref{algo2})}

Within initial cluster $C_k$ with representatives $\mathcal{R}_k$, the
reassignment loop (Algorithm~\ref{algo2}) examines, for each
$i \in C_k$, only the columns of $S$ indexed by $\mathcal{R}_k$. The
cost is $O(n_k \, |\mathcal{R}_k|)$ per cluster and
$$
T_{\text{Reassign}} \;=\; O\!\left(\sum_{k=1}^{K} n_k \, |\mathcal{R}_k|\right)
\;\le\; O\!\left(\sum_{k=1}^{K} n_k^2\right),
$$
which is dominated by the pairwise similarity computation and does not
increase the asymptotic complexity. No additional pairwise storage is
needed because $S$ and $M$ have already been materialized in Stage~2.
This matches the empirical observation that the reassignment variant
has essentially identical runtime to the base algorithm (Table~\ref{tab:sim},
``Time'' column).

\subsection{Overall Complexity}

Combining Stage~1 and Stage~2,
\begin{align*}
T_{\text{total}} &\;=\; O\!\left( T_1 \, b \, K \, d \;+\; d \sum_{k=1}^{K} n_k^2 \right), \\
M_{\text{total}} &\;=\; O\!\left( n \, d \;+\; \max_{k} n_k^2 \right).
\end{align*}
For roughly balanced initial clusters with $n_k \approx n/K$,
$$
T_{\text{total}} \;=\; O\!\left(n d + \tfrac{n^2 d}{K}\right),
\qquad
M_{\text{total}} \;=\; O\!\left(n d + \tfrac{n^2}{K^2}\right).
$$
In the regime $K = \Theta(n)$ (i.e., $n_k = O(1)$), Stage~2 becomes
$O(n \, d)$, matching Stage~1, and the overall algorithm is linear in
$n$.

\subsection{Comparison with Baselines}
\label{app:comp-comp}

For reference, the baseline methods compared in Section \ref{sec:sim} have the
following standard complexities on $n$ points in $\mathbb{R}^d$:
\begin{itemize}
  \item \textbf{K-Means} (Lloyd): $O(n K d T)$ time, $O(n d)$ memory.
  \item \textbf{Mini-batch K-Means}: $O(b K d T)$ time, $O(n d)$ memory.
  \item \textbf{Agglomerative (Ward)}: $O(n^2 d)$ time, $O(n^2)$
  memory (storing the full distance matrix or heap).
  \item \textbf{BIRCH}: $O(n d)$ amortized time, but with large
  constants.
  \item \textbf{Spectral Clustering}: $O(n^3)$ time in the dense case,
  with $O(n^2)$ memory for the affinity matrix.
  \item \textbf{Gaussian Mixture (EM)}: $O(n K d^2 T)$ time, $O(n d +
  K d^2)$ memory.
\end{itemize}
The threshold-based methods most related to ours (Appendix
\ref{app:threshold-comparison}) share the same bottleneck, as each computes the
full pairwise similarity matrix over all $n$ samples:
\begin{itemize}
  \item \textbf{Star Clustering}: $O(n^2 d)$ time to compute the pairwise
  similarity matrix and $O(n^2)$ memory; the star-cover step is then linear in
  the graph size.
  \item \textbf{SimClus}: $O(n^2 d)$ time and $O(n^2)$ memory to compute the
  similarity matrix, plus greedy set-cover selection over all $n$ samples. In
  practice this global selection makes it the slowest and most memory-intensive
  method we benchmark (Section \ref{sec:sim}).
  \item \textbf{Community Detection}: $O(n^2 d)$ time and $O(n^2)$ memory for the
  full pairwise similarity matrix. To scale, SentenceTransformers batches the data and runs
  Community Detection within each batch of size $b$, reducing cost to
  $O(n b d)$ time and $O(n b)$ memory. This batched variant is effectively a
  special case of our two-stage design in which the initial ``clustering'' is a
  random partition into fixed-size batches; because random batching separates
  similar samples more readily than Mini-batch K-Means, it tends to produce more
  final clusters (weaker data reduction) than an initial clustering that groups
  similar samples together.
\end{itemize}
None of the average-metric baselines admit a user-specified minimal within-cluster
similarity constraint, and the $O(n^2)$ memory requirement of Agglomerative,
Spectral, Star Clustering, SimClus, and (unbatched) Community Detection is
prohibitive at the $n = 38$M scale (as noted in Section~\ref{sec:exp}:
$\sim 90$\,TB for $n = 5$M). Algorithm~\ref{algo1} circumvents this by restricting
the quadratic computation to \emph{within} each initial cluster, yielding
$O(\max_k n_k^2) \ll O(n^2)$ memory for any $K \gg 1$. This is the
underlying reason for the up-to-$1000\times$ runtime and memory advantage
observed in Section~\ref{sec:sim}.

\section{Additional Experiment Results}
\label{app:add_exp}
\subsection{Cluster Size Distribution Additional Figure}
Refer to Figure \ref{fig:size2} for cluster size distribution results on Cosmopedia, complementary to Figure \ref{fig:size}. Across all three datasets, the greedy coverage-based methods (SGC and SimClus) produce the most skewed distributions, in which the largest clusters cover most samples, enabling effective tail-cluster trimming to further reduce data size and downstream computations; the average-metric baselines, which target balanced clusters, are the least skewed.

\begin{figure}[t]
  \centering
  \includegraphics[width=\columnwidth]{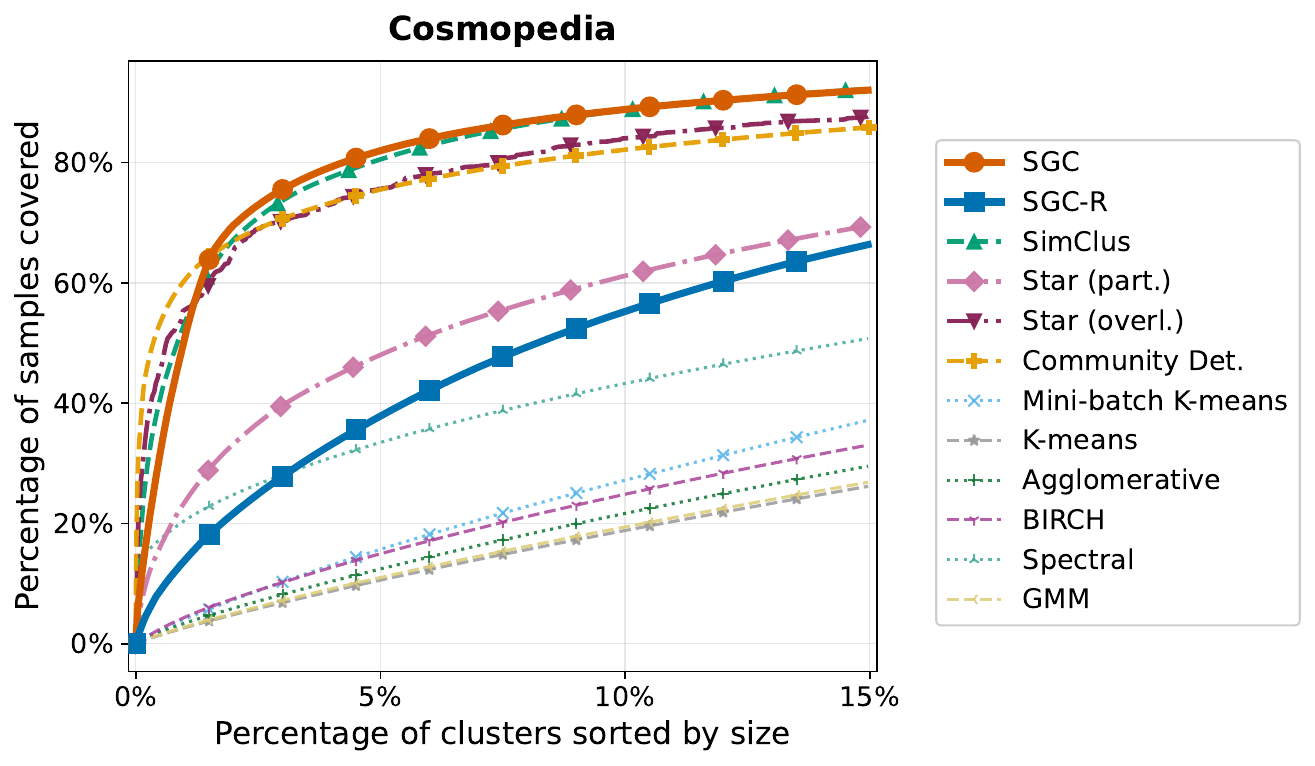}
  \caption{Sample coverage by sorted clusters on Cosmopedia.}
  \label{fig:size2}
\end{figure}

\subsection{Impact of Minimal Similarity and Household Attributes}
\label{app:dedup_fold}
There is a tradeoff between within-cluster similarity and the number of clusters. In the case of personalized recommendations, the higher the similarity, the more likely we recommend the relevant products to all customers in a cluster, at the cost of increasing the number of clusters and thus downstream computations. Similarly, the number of clusters will increase if we require matching household attributes in addition to the persona similarity, such as adult gender, adult count, child presence, child gender, and child age. These household attributes are important for personalized recommendations and not always captured in the customer shopping personas. We show how the number of clusters is affected by the minimal similarity and matching attributes required between each sample and the cluster representative in Table \ref{tab:sim-attr} on randomly selected 5 million customers of the 38 million customer dataset.

\begin{table}[h]
  \caption{Clustering performance metrics with different matching attributes and similarity thresholds on 5 million customers.}
  \label{tab:sim-attr}
  \centering
  \adjustbox{width=1.0\textwidth}{
  \begin{tabular}{llrrrr}
    \toprule
    \begin{tabular}[c]{@{}l@{}}Matching Attributes\\Required\end{tabular} &
    \begin{tabular}[c]{@{}l@{}}Minimal\\Similarity\\Required\end{tabular} &
    \begin{tabular}[c]{@{}r@{}}Number of\\Clusters\end{tabular} &
    \begin{tabular}[c]{@{}r@{}}Average\\Similarity\end{tabular} &
    \begin{tabular}[c]{@{}r@{}}Fraction of Data\\Retained\end{tabular} &
    \begin{tabular}[c]{@{}r@{}}Data Reduction\\Fold\end{tabular} \\
    \midrule
    \multirow{6}{*}{None}
    & 0.65 & 2344 & 0.794 & 0.05\% & 2133$\times$ \\
    & 0.7 & 5012 & 0.812 & 0.1\% & 998$\times$ \\
    & 0.75 & 26868 & 0.824 & 0.5\% & 186$\times$ \\
    & 0.8 & 230262 & 0.840 & 4.6\% & 22$\times$ \\
    & 0.85 & 1694033 & 0.904 & 33.9\% & 3$\times$ \\
    & 0.9 & 4695169 & 0.994 & 93.9\% & 1$\times$ \\
    \midrule
    \multirow{6}{*}{\begin{tabular}[c]{@{}l@{}}adult gender,\\adult count, child\\presence, child\\gender, child age\end{tabular}}
    & 0.65 & 44285 & 0.792 & 0.9\% & 113$\times$ \\
    & 0.7 & 49804 & 0.809 & 1.0\% & 100$\times$ \\
    & 0.75 & 97298 & 0.824 & 1.9\% & 51$\times$ \\
    & 0.8 & 435727 & 0.846 & 8.7\% & 11$\times$ \\
    & 0.85 & 2146887 & 0.916 & 42.9\% & 2$\times$ \\
    & 0.9 & 4787855 & 0.996 & 95.8\% & 1$\times$ \\
    \bottomrule
  \end{tabular}
  }
\end{table}

\subsection{Impact of Hyperparameters and Robustness}
\label{app:robust}
We perform SGC-R on the 100K Shopping Personas for different hyperparameters of Mini-batch K-Means used for the initial clustering before representative selection. Because the greedy selection and optional reassignment in the second stage help further partition the initial clusters based on desired safety/quality constraints, the proposed clustering method is robust to the hyperparameters of Mini-batch K-Means or even the choice of the clustering method used in the first stage (Table \ref{tab:robustness}).

There are two main tradeoffs in our clustering application. The first main tradeoff is between cluster quality and data reduction fold, which impacts the recommendation quality and downstream computation savings. This tradeoff is controlled by the user specified minimal similarity threshold and matching attributes shown in Table \ref{tab:sim-attr}. The second main tradeoff is between scalability (in terms of latency and memory usage) and data reduction fold. This tradeoff is controlled by the initial number of clusters, $K$. Recall that in the second stage of the proposed clustering we compute pairwise embedding similarity within each initial cluster for the greedy representative selection. This pairwise computation is $O(n^2)$ w.r.t. sample size $n$, in this case the size of each initial cluster. Thus as $K$ increases, the latency and memory requirement will both decrease, increasing scalability. In exchange, the initial clustering also introduces noise so that samples that pass the final cluster criteria might end up separated during initial clustering, reducing the efficiency of data size reduction, which worsens as $K$ increases (Table \ref{tab:robustness}). In practice we recommend choosing a $K$ as small as latency and memory requirements allow to maximize data reduction.

\begin{table}[t]
  \caption{Robustness analysis of SGC-R on internal Customer Shopping Personas across different Mini-Batch K-Means (MBKM) hyperparameters: batch size, init size, and number of initial clusters. We report average within-cluster similarity (Avg.\ Sim.), minimal within-cluster similarity (Min.\ Sim.), number of final clusters (\# Clusters), data reduction fold (Reduc.\ Fold), Stage~1 time (MBKM), and Stage~2 time (Representative Selection and Reassignment) in seconds.}
  \label{tab:robustness}
  \centering
  \adjustbox{width=1.0\textwidth}{
  \begin{tabular}{rrrrrrrrr}
    \toprule
    \begin{tabular}[c]{@{}r@{}}MBKM\\Batch Size\end{tabular}
    & \begin{tabular}[c]{@{}r@{}}MBKM\\Init Size\end{tabular}
    & \begin{tabular}[c]{@{}r@{}}Num Initial\\Clusters ($K$)\end{tabular}
    & \begin{tabular}[c]{@{}r@{}}Avg.\\Sim.\end{tabular}
    & \begin{tabular}[c]{@{}r@{}}Min.\\Sim.\end{tabular}
    & \begin{tabular}[c]{@{}r@{}}\# Final\\Clusters\end{tabular}
    & \begin{tabular}[c]{@{}r@{}}Reduc.\\Fold\end{tabular}
    & \begin{tabular}[c]{@{}r@{}}Stage 1\\Time (sec.)\end{tabular}
    & \begin{tabular}[c]{@{}r@{}}Stage 2\\Time (sec.)\end{tabular} \\
    \midrule
    1024  & 1024  & 10  & 0.826 & 0.750 & 2055 & 48$\times$ & 0.12 & 12.97 \\
    1024  & 1024  & 50  & 0.825 & 0.750 & 2539 & 39$\times$ & 0.22 & 1.48  \\
    1024  & 1024  & 250 & 0.824 & 0.750 & 2920 & 34$\times$ & 0.46 & 1.02  \\
    \midrule
    1024  & 10240 & 10  & 0.825 & 0.750 & 2029 & 49$\times$ & 0.18 & 14.59 \\
    1024  & 10240 & 50  & 0.825 & 0.750 & 2575 & 38$\times$ & 0.46 & 1.41  \\
    1024  & 10240 & 250 & 0.824 & 0.750 & 3042 & 32$\times$ & 1.57 & 0.90  \\
    \midrule
    10240 & 1024  & 10  & 0.826 & 0.750 & 2034 & 49$\times$ & 0.29 & 13.54 \\
    10240 & 1024  & 50  & 0.825 & 0.750 & 2579 & 38$\times$ & 0.29 & 1.44  \\
    10240 & 1024  & 250 & 0.824 & 0.750 & 3054 & 32$\times$ & 0.54 & 0.89  \\
    \midrule
    10240 & 10240 & 10  & 0.826 & 0.750 & 2048 & 48$\times$ & 0.21 & 14.89 \\
    10240 & 10240 & 50  & 0.825 & 0.750 & 2602 & 38$\times$ & 0.58 & 1.46  \\
    10240 & 10240 & 250 & 0.823 & 0.750 & 3104 & 32$\times$ & 1.69 & 0.86  \\
    \bottomrule
  \end{tabular}
  }
\end{table}

\section{Experimental Setup Details}
\label{app:hyp}
In the experiment on within-cluster similarities from Section~\ref{sec:sim}, we specify the minimal similarity thresholds (0.75 on internal shopping personas, 0.3 on AG News, and 0.4 on Cosmopedia) to obtain $\sim$50 fold reduction in data size based on our target use case in personalized recommendation and for making results across 3 datasets more comparable. The conclusions are consistent for other values of minimal similarity thresholds based on further experiments. We set the number of initial clusters to be 0.04\% of the sample size. Increasing this ratio will reduce the memory requirement of pair-wise similarity computation within each initial cluster at the potential cost of increasing the number of final clusters.

We select embedding models from the SentenceTransformer model zoo \citep{reimers-2019-sentence-bert} by benchmark performance, inference speed, and context-window limit. For shopping personas and AG News, whose texts are short, we use \texttt{all-\allowbreak MiniLM-\allowbreak L6-\allowbreak v2} (384-dimensional embeddings, 256-token context), which is fast and performs well across benchmarks. Cosmopedia articles are substantially longer, so we instead use \texttt{all-\allowbreak distil\-\allowbreak roberta-\allowbreak v1} (768-dimensional embeddings), which offers a larger context window (512 tokens) to avoid truncating long inputs at a modest speed cost. Any embedding model meeting these criteria can be substituted; the same model is used for the proposed method and all baselines on a given dataset for a fair comparison.

We use the default hyperparameters of the clustering methods benchmarked against as implemented in Scikit-learn \citep{scikit-learn} 1.7.2, except for Spectral Clustering whose hyperparameters are slightly adjusted for memory efficiency. The number of clusters is set to match the final number of clusters produced by SGC on each dataset to achieve the same data size reduction. Other key hyperparameters are as follows:
\begin{enumerate}
    \item \textbf{Mini-batch K-Means}: `max\_iter': 100, `batch\_size': 10240, `init\_size': 30720, `random\_state': 123.
    \item \textbf{K-Means}: `init': `k-means++', `n\_init': `auto', `max\_iter': 300, `algorithm': `lloyd'.
    \item \textbf{Agglomerative Clustering}: `metric': `euclidean', `linkage': `ward', `compute\_full\_tree': `auto'.
    \item \textbf{BIRCH}: `threshold': 0.5, `branching\_factor': 50.
    \item \textbf{Spectral Clustering}: `affinity': `nearest\_neighbors', `n\_neighbors': 10, `eigen\_solver': `lobpcg'.
    \item \textbf{Gaussian Mixture}: `covariance\_type': `full', `init\_params': `kmeans', `max\_iter': 100.
\end{enumerate}

\end{document}